\documentclass[10pt,twocolumn,letterpaper]{article}

\usepackage{cvpr}
\usepackage{times}
\usepackage{epsfig}
\usepackage{graphicx}
\usepackage{amsmath}
\usepackage{amssymb}
\usepackage{amsthm}

\usepackage{enumitem}
\DeclareMathOperator*{\argmin}{arg\,min}
\DeclareMathOperator*{\argmax}{arg\,max}

\newtheorem{prop}{Property}

\usepackage{algorithm}
\usepackage{algorithmic}
\usepackage[ruled,vlined,algo2e]{algorithm2e}

\usepackage[pagebackref=true,breaklinks=true,letterpaper=true,colorlinks,bookmarks=false]{hyperref}

\cvprfinalcopy 

\def\cvprPaperID{677} 

\ifcvprfinal\fi
\begin{document}

\title{Learning Joint Feature Adaptation for Zero-Shot Recognition}

\author{Ziming Zhang\\
		Mitsubishi Electric Research Laboratories \\
		201 Broadway, Cambridge, MA 02139-1955 \\
		{\tt\small zzhang@merl.com}
		\and
		Venkatesh Saligrama\\
		ECE, Boston University \\
		8 Saint Mary's Street, Boston, MA 02215\\
		{\tt\small srv@bu.edu}
}

\maketitle

\begin{abstract}
	
	Zero-shot recognition (ZSR) aims to recognize target-domain data instances of unseen classes based on the models learned from associated pairs of seen-class source and target domain data. One of the key challenges in ZSR is the relative {\em scarcity} of source-domain features (\eg one feature vector per class), which do not fully account for wide variability in target-domain instances.
	
	In this paper we propose a novel framework of learning data-dependent feature transforms for scoring similarity between an arbitrary pair of source and target data instances to account for the wide variability in target domain. Our proposed approach is based on optimizing over a parameterized family of local feature displacements that maximize the source-target adaptive similarity functions. Accordingly we propose formulating zero-shot learning (ZSL) using {\em latent structural SVMs} to learn our similarity functions from training data. As demonstration we design a specific algorithm under the proposed framework involving bilinear similarity functions and regularized least squares as penalties for feature displacement. 
	We test our approach on several benchmark datasets for ZSR and show significant improvement over the state-of-the-art. For instance, on aP\&Y dataset we can achieve 80.89\% in terms of recognition accuracy, outperforming the state-of-the-art by 11.15\%.
	
\end{abstract}

\section{Introduction}
While there has been significant progress on supervised large-scale classification in recent years \cite{ILSVRCarxiv14}, the lack of sufficient annotated training data uniformly across  all classes \cite{Bhatia15,antol2014zero} has been a bottleneck in achieving acceptable performance. At a basic level, in these cases we encounter situations where we may have sufficient annotated training data for some of the classes and little or even no annotated data to train supervised classifiers for the other interesting classes. In this context, a fundamental question that arises is as to how to leverage training data for observed classes for recognition of rare or unobserved classes. 

One possible scenario is when we have data from a different domain that can be collected easily, as is assumed to be the case in zero-shot recognition (ZSR). In ZSR we are given {\em source} and {\em target} domains as training data belonging to a sub-collection of classes, forming {\em seen} or observed classes. No training data is available for other {\em unseen} classes. The class information in source domain is described in a variety of ways such as attribute vectors \cite{farhadi2009attribute,10.1109/TPAMI.2013.140,mensink2012metric,Parikh:2011:IBD:2191740.2191861,rohrbach2011largeScale}, language words/phrases \cite{Berg:2010:AAD:1886063.1886114,frome2013devise,socher2013zero}, or even learned classifiers \cite{yu2013designing}. Target domain is described by a joint distribution of data (\eg images or videos) and labels \cite{10.1109/TPAMI.2013.140, wu2014zero}. 

\begin{figure}[t]
	\centerline{\includegraphics[width=\linewidth]{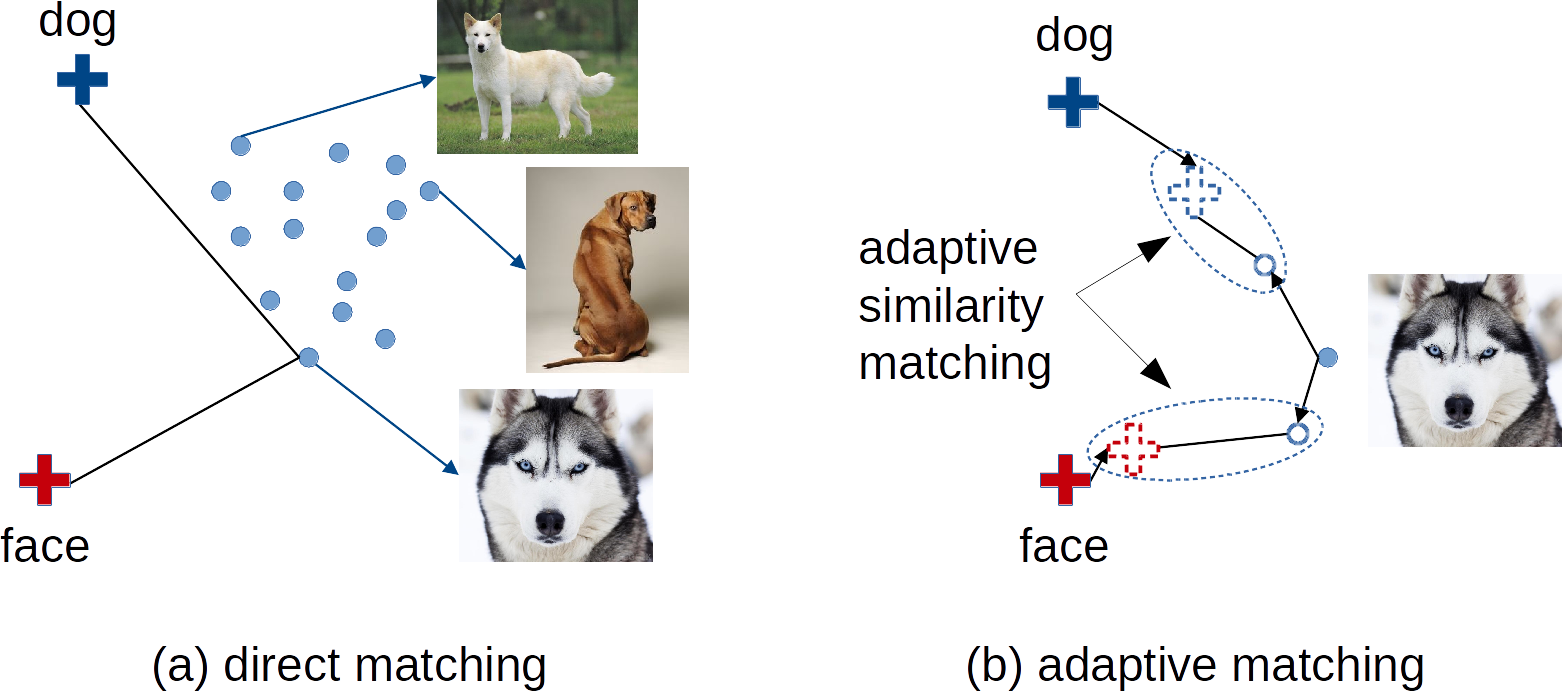}}
	\vspace{1mm}
	\caption{\footnotesize Illustration of our intuition behind learning joint data-dependent feature adaptation for ZSR. Here $+$ and $\circ$ denote source and target domain data embeddings, respectively, colors denote different classes, filled/empty shapes denote original/adapted feature embeddings. Using the direct matching in (a) the image will be mis-classified as ``face'' (based on distance measure), while using the adaptive matching in (b) it will be classified correctly as ``dog''.}\label{fig:idea}
	\vspace{-4mm}
\end{figure}

The challenge in ZSR lies in learning models based on seen-class data that can generalize to unseen classes. In this context, many perspectives for zero-shot learning (ZSL) have been proposed, including unseen classifier prediction from source domain data based on learning attribute classifiers~\cite{lampert2009attribute}, learning similarity functions between source and target domains to score similarity for unseen classes~\cite{akata2013label}, and manifold embedding methods based on identifying inter-class relationships in source and target data that can be aligned during test time~\cite{Zhang2015,zhang2016zero,Changpinyo_2016_CVPR}. 

Nevertheless, the challenge posed by the relative sparseness of source domain descriptions in recognition has not been fully considered. In particular, the target-domain data exhibits significant intra-class variation (\eg appearance and poses). On the other hand the source domain information is relatively sparse and typically amounting to a single attribute vector. This is generally insufficient to account for all of the intra-class variation. 

Fig. \ref{fig:idea} illustrates this point. In the joint embedded feature space, the target-domain data distributions of certain classes (\eg ``dog'' class in the figure) are relatively flat and consist of data instances with large variation. This issue leads us to view the presented source domain vector as a ``mean-value'' over all candidate (or alternative) source vectors. During test time for a given target instance we optimize the matching over all possible source and target candidates in a neighborhood of the presented source and target instances. During training we propose learning a data-dependent feature transform chosen from a parameterized family of displacement functions that maximizes the similarity between an arbitrary source and target instances.
%
We learn our similarity functions from training data using {\em latent structural SVMs}. As demonstration we design a specific algorithm under the proposed framework involving bilinear similarity functions and regularized least squares as penalties for displacement.

To illustrate how this would work, consider again Fig.~\ref{fig:idea}. In test time our proposed approach manifests as new features (\ie empty ``$+$'' in the figure) that adapt to the potential contents in target data instance. This leads to significantly richer representations than the provided source-domain vectors. Our proposed approach also induces displacements in the target domain data instances. These displaced features (\ie empty ``$\circ$'' in the figure) in turn adapt to source domain features. This process is akin to de-noising of presented target-domain data/features. As we see, by using the new features, the dog face image is correctly classified based on similarity measure between the new data-dependent adapted features, as illustrated in Fig. \ref{fig:idea}(b).

\noindent
{\bf Contributions:} 
In this paper we introduce a novel {\em adaptive similarity function} for comparing an arbitrary pair of source and target domain data instances. This function, in test time and adaptively in a data-dependent way, determines the similarity between presented source and target instances. 

We propose considering optimizing over a parameterized bilinear family of functions for our cross-domain similarity measure. Alternating optimization is utilized to efficiently estimate (globally) best adapted features within a constrained family of displacements. In this context we show that the compatibility function defined in \cite{akata2013label, Akata2015} is indeed a special case of our similarity function.

To learn the parameters in our adaptive similarity function, we further propose formulating the ZSL problem using {\em latent structural SVMs}. The latent part comes from the adapted (latent) features, which are considered as the latent variables in the formulation. The structural part arises from the structures of label embeddings as did in~\cite{akata2013label, Akata2015}.

We test our approach on four benchmark image datasets for ZSL/ZSR, namely, aP\&Y, AwA, CUB and SUN-attribute. Under both standard and transductive settings, our approach outperforms the state-of-the-art significantly.

\subsection{Related Work}
In general ZSL/ZSR approaches can be divided into two categories: standard setting and transductive setting. Recently zero-shot approaches have been successfully applied to several visual tasks such as event detection \cite{wu2014zero, chang2015semantic, elhoseiny2015zero}, action recognition \cite{gan2015exploring}, and image tagging \cite{Zhang_2016_CVPR}. Below we primarily describe learning approaches in this context.

\noindent
{\bf Standard Setting:} In test time, the source-domain descriptors for unseen classes are all given at once. Our task is to sequentially recognize target-domain instances as they are revealed {\em one at a time}. 

In this context, several works in the literature are based on training attribute classifiers which directly map target-domain data into source-domain attribute space \cite{palatucci2009zero, 10.1109/TPAMI.2013.140, mahajan2011joint, wang2013unified, yu2013designing, yu2010attribute, mensink2014costa, hariharan2012efficient, Romera-Paredes2015, Al-Halah_2016_CVPR}. The resulting attribute classifiers do not fully account for data noise in source (\eg ambiguity or mislabeling in attributes) and target (\eg large variation because of the changes of appearance, poses, \etc) domains. 

Linear and nonlinear embedding approaches \cite{akata2013label, Akata2015,frome2013devise,norouziMBSSFCD14,socher2013zero,Li2015, Li_ICCV2015,Qiao_2016_CVPR,Ba2015, Kodirov2015, Zhang2015, zhang2016zero, Changpinyo_2016_CVPR,Xian2016Latent,bucher2016improving,wang2016relational} have attracted attention recently. The basic idea of these methods is to embed the source and target domain features into a Kronecker product embedding space. For instance, Akata \etal \cite{akata2013label, Akata2015} proposed label embedding to map class labels into a high dimensional vector space (\eg source-domain attribute space), and measure cross-domain similarities using a bilinear function whose parameters are learned using structured SVMs. Zhang and Saligrama \cite{zhang2016zero} proposed a joint learning framework to learn the latent embeddings for both domains and utilized them for similarity measure. Changpinyo \etal \cite{Changpinyo_2016_CVPR} proposed a learning method to generate synthesized classifiers for unseen classes. Bucher \etal \cite{bucher2016improving} proposed a metric learning based formulation to improve semantic embedding consistency, achieving the best performance on the four benchmark datasets under the standard setting in the current literature, to our best knowledge. The underlying assumption behind such approaches is that there exist (hidden) corresponding matches between source-domain feature vectors and target-domain data distributions, \eg one-to-one match \cite{akata2013label,Akata2015,zhang2016zero} or one-to-many match \cite{Xian2016Latent}. In this context there are other related proposed methods such as semantic transfer propagation \cite{conf/nips/RohrbachES13}, random forest based approaches \cite{jayaraman2014unreliable}, semantic manifold distance \cite{fu2015} approaches, and similarity calibration method \cite{chao2016empirical}. Nevertheless, the issue of source-domain sparsity and the resulting imbalance with target-domain data is not fully accounted for in these methods.

Our proposed method explicitly focuses on handling the scarcity issue of source-domain data by learning data-dependent latent features. This in turn accounts for the large data variation in target domain implicitly so that the cross-domain matches can be improved.

\noindent
{\bf Transductive Setting:} Recently researchers have begun to incorporate test-time unseen-class data in target domain into ZSL/ZSR as unlabeled data analogous to the transductive setting. This has led to approaches that attempt to account for domain shift \cite{Kodirov2015,embedding2014ECCV,fu2015transductive,guo2016transductive,zhang2016ECCV}. In this setting, during test time, we are given a list of all unlabelled target instances in addition to unseen-class source-domain descriptions. Potentially these methods can be used in conjunction with any similarity learning procedure trained on seen-class data, as demonstrated in \cite{zhang2016ECCV}, to score similarity between unseen classes and target domain data instances. 

While much of the focus of this paper is on the standard setting, in our experimental section we also test our learning algorithm in the transductive mode to benchmark our performance in the transductive setting. 

\section{Our Approach}

\subsection{ZSL/ZSR Problem Setup}
In the training stage, we are given a set of observed classes $\mathcal{L}_o$. For source domain, attribute vectors (or label embeddings) in the form of $\{\psi(y)\},\,\,\forall y\in\mathcal{L}_o$, are provided. Typically there exists only one vector per class. Corresponding target domain data instances $x \in \mathcal{X}$ and feature embeddings $\phi(x)$ associated with the observed source labels are also provided for training. 
We aggregate training data as $\mathbb{O}=\{(\phi(x_i),\psi(y_i),y_i),\,\forall i \in \mathbb{T}\}$, where $i$ denotes the $i$-th training data instance in target domain. 

Our goal is to learn a prediction model, by leveraging observed training data, $\mathbb{O}$, such that it generalizes well to unobserved data instances and classes during test time.

In the testing stage, a set of source vectors corresponding to unobserved classes $\mathcal{L}_u$ are revealed. For a given unobserved data instance $\bar{x}$ from target domain, the task is to identify the source vector among those unobserved classes that corresponds to $\bar{x}$. Abstractly, our decision rule is based on maximizing a posterior probability (MAP) conditioned on all the available data:
\begin{align}\label{eqn:y*}
\bar{y}^{*} = \argmax_{\bar{y}\in \mathcal{L}_u } \mathbb{P}_{\psi,\phi}\left(\bar{y}\mid \bar{x};\mathbb{O}\right), \;
\end{align}
where $\mathbb{P}_{\psi,\phi}(\cdot)$ denotes the posterior probability tuned to the embeddding functions $\psi, \phi$. In what follows we drop the parameter dependence on $\psi, \phi$ for notational simplicity, since we assume that these embedding functions are provided a priori. The posterior probability is unknown and must be learned from training data. We describe our proposed approach in the following section.

\subsection{General Learning Framework}
\subsubsection{Parameterized Family of Posterior Distributions}
We face two fundamental challenges in ZSR. 

First, target instances and labels for unobserved classes are not known during training. Therefore, proposed methods must base its recognition on scoring the similarity between an arbitrary source descriptor and a target instance. 

Second, source vectors in ZSR are sparse and typically we only observe a single source vector per class. On the other hand there is significant variability in the target domain. Consequently, the source vectors serve only as ``average'' attribute descriptors across the target domain instances. The source descriptor that best matches a target instance is a vector that is typically close to but not necessarily equal to the given source domain vector. We propose optimizing over all such vectors in both learning and test time to determine the optimal matching source descriptors. 

In this context we propose a family of posterior distributions:
%
%
To account for relative sparseness of source domain descriptors and large variability of target domain instances we introduce new data-dependent feature vectors $\mathbf{z}_s, \mathbf{z}_t$ corresponding to source and target domains, respectively. To ensure that these feature vectors are ``typically'' close to the given source and target data pair we introduce a displacement penalty term $d_{\boldsymbol{\omega}}(x, y, \mathbf{z}_s, \mathbf{z}_t)$ parametrized by $\boldsymbol{\omega}$. To score similarity between source and target domain data we propose a scoring function, $s_{\mathbf{W}}(\mathbf{z}_s, \mathbf{z}_t)$, that scores similarity between the new data-dependent feature vectors parameterized by a matrix $\mathbf{W}$. This leads to the following posterior probability:
\begin{align}
\mathbb{P}(y,\mathbf{z}_s,\mathbf{z}_t \mid x;\mathbf{W},\boldsymbol\omega) \propto \exp(s_{\mathbf{W}}(\mathbf{z}_s, \mathbf{z}_t)-d_{\boldsymbol{\omega}}(x, y, \mathbf{z}_s, \mathbf{z}_t)).
\end{align}

In order to compute the posterior $\mathbb{P}(y|x)$ we can marginalize $\mathbb{P}(y,\mathbf{z}_s,\mathbf{z}_t \mid x;\mathbf{W},\boldsymbol\omega)$ over variables $\mathbf{z}_s\in\mathcal{Z}_s, \mathbf{z}_t\in\mathcal{Z}_t$, where $\mathcal{Z}_s, \mathcal{Z}_t$ denote their corresponding feasible domains (\eg simplex). However, in general this calculation will be very difficult given arbitrary parameter spaces, and typically Bayesian parametrization is often involved (\eg \cite{ping2014marginal}) to simplify the calculation. Alternatively the posterior can be upper-bounded by the maximum value over the variables, as did in \cite{zhang2016zero}, which can be very computationally efficient and demonstrated with good performance for ZSR as well.

Therefore, here we adopt the strategy in \cite{zhang2016zero} and take the maximum for posterior approximation purpose. This leads naturally to our {\em adaptive similarity function} as below for scoring each target data instance with a class label:
\begin{align}\label{eqn:f}
& f(x, y; \mathbf{W}, \boldsymbol{\omega}) \stackrel{\Delta}{=} \max_{\mathbf{z}_s\in\mathcal{Z}_s, \mathbf{z}_t\in\mathcal{Z}_t} \mathbb{P}(y,\mathbf{z}_s,\mathbf{z}_t \mid x;\mathbf{W},\boldsymbol\omega) \nonumber \\
& = \max_{\mathbf{z}_s\in\mathcal{Z}_s, \mathbf{z}_t\in\mathcal{Z}_t}\Big\{s_{\mathbf{W}}(\mathbf{z}_s, \mathbf{z}_t) - d_{\boldsymbol{\omega}}(x, y, \mathbf{z}_s, \mathbf{z}_t)\Big\}.
\end{align}

Intuitively our similarity function allows the features to move from their original locations in the feature space (\ie adaptation) to achieve a higher similarity score within a neighborhood (feature displacements incur penalties). Our function in Eq. \ref{eqn:f} thus attempts to achieve a balance between these two objectives. In fact similar strategy has been widely used in deformable part models (DPM) \cite{felzenszwalb2010object}, where 2D locations for parts are considered as adapted features.

\subsubsection{Learning with Latent Structural SVMs}
The parameters $\mathbf{z}_s, \mathbf{z}_t$ in Eq. \ref{eqn:f} play the role of latent variables for given values of $\mathbf{W}, \boldsymbol{\omega}$. Consequently, we can pose the problem as a latent structural SVM problem by viewing the label variable $y$ as taking values from a structured output space:
%
\begin{align}\label{eqn:general_learning}
& \min_{\mathbf{W}\in\mathcal{W}, \boldsymbol{\omega}\in\Omega, \boldsymbol{\xi}}  \mathcal{R}_1(\mathbf{W}) + \mathcal{R}_2(\boldsymbol{\omega}) + \sum_i \xi_i \\ 
& \mbox{s.t.} \;  f(x_i, y_i; \mathbf{W}, \boldsymbol{\omega}) - f(x_i, y; \mathbf{W}, \boldsymbol{\omega}) \geq \Delta(y_i, y) - \xi_i, \nonumber\\
& \hspace{5mm} \forall i, \xi_i\geq0,  x_i\in\mathcal{X}, y_i, y\in\mathcal{L}_o, \nonumber
\end{align}
where $\mathcal{R}_1, \mathcal{R}_2$ denote two regularization functions (\eg $\ell_2$-norm regularizers) for parameters $\mathbf{W}, \boldsymbol{\omega}$, respectively, $\Delta$ denotes a penalty term measuring the difference between the ground-truth label $y_i$ and an arbitrary label $y$, $\mathcal{W}, \Omega$ denote the feasible domains for $\mathbf{W}, \boldsymbol{\omega}$, respectively, and $\xi_i, \forall i$ is a slack variable. The cutting-plane algorithm \cite{yu2009learning} can be used for general training purpose. 

In test time, we replace the probability term in Eq. \ref{eqn:y*} with our adaptive similarity function in Eq. \ref{eqn:f} to rewrite the decision rule for ZSR as follows:
\begin{align}\label{eqn:y**}
\bar{y}^{*} = \argmax_{\bar{y}\in\mathcal{L}_u} f(\bar{x}, \bar{y}; \mathbf{W}, \boldsymbol{\omega}), \; \forall \bar{x}.
\end{align}

\subsection{Bilinear Adaptive Similarity Functions}
For the purpose of demonstration we describe one instance of an adaptive similarity function that can be utilized in our general learning framework.

Specifically we design the similarity term  $s_{\mathbf{W}}(\mathbf{z}_s, \mathbf{z}_t)$ in Eq. \ref{eqn:f} as a bilinear function. These type of functions have been widely used in recent ZSL literature, \eg \cite{akata2013label,Akata2015,zhang2016zero}, and has been shown to achieve state-of-the-art performance. For the penalty term, we simply adopt the regularized least square loss for the displacement. Putting these together, we propose the following adaptive similarity function:
\begin{align}\label{eqn:ff}
& f(x, y; \mathbf{W}, \boldsymbol{\omega}) = \max_{\mathbf{z}_s\in\mathcal{Z}_s, \mathbf{z}_t\in\mathcal{Z}_t}\left\{\mathbf{z}_t^T\mathbf{W}\mathbf{z}_s - \frac{\omega_1}{2}\left\|\mathbf{z}_t-\phi(x)\right\|_2^2 \right. \nonumber \\
& \hspace{5mm} \left. - \frac{\omega_2}{2}\left\|\mathbf{z}_s-\psi(y)\right\|_2^2 - \frac{\omega_3}{2}\|\mathbf{z}_t\|_2^2 - \frac{\omega_4}{2}\|\mathbf{z}_s\|_2^2 \right\},
\end{align}
where $\mathbf{W}\in\mathbb{R}^{d_t\times d_s}$ is a weighting matrix between $\mathbf{z}_t\in\mathbb{R}^{d_t}$ and $\mathbf{z}_s\in\mathbb{R}^{d_s}$, $\boldsymbol{\omega}=[\omega_1; \omega_2; \omega_3; \omega_4]$ is a 4D vector controlling the trade-off between similarity and penalty, and $\|\cdot\|_2$ denotes the $\ell_2$ norm of a vector. In general we can utilize {\em alternating optimization} (AO) to solve Eq.~\ref{eqn:ff} as follows:
\begin{align}
& \mathbf{z}_t = \argmin_{\mathbf{z}\in\mathcal{Z}_t} \left\{\omega_{13}\left\|\mathbf{z} - \left(\frac{\omega_1\phi(x)}{\omega_{13}}+\frac{\mathbf{W}\mathbf{z}_s}{\omega_{13}}\right) \right\|_2^2\right\}, \label{eqn:zt}\\
& \mathbf{z}_s = \argmin_{\mathbf{z}\in\mathcal{Z}_s} \left\{\omega_{24}\left\|\mathbf{z} - \left(\frac{\omega_2\psi(y)}{\omega_{24}}+\frac{\mathbf{z}_t^T\mathbf{W}}{\omega_{24}}\right) \right\|_2^2\right\}, \label{eqn:zs}
\end{align}
where $\omega_{13}=\omega_1+\omega_3$ and $\omega_{24}=\omega_2+\omega_4$.

Ideally, we would like to have a decision rule that during test time using Eq. \ref{eqn:ff} converges to a (unique) global solution\footnote{In this paper we only focus on utilizing convex optimization to achieve such global solutions, although minimizing concave functions over nonempty closed convex sets may result in global solutions as well, but it is much harder to be solved and, more importantly, without any guarantee.} for an arbitrary pair of source and target data instances. This is because we can then be certain that the similarity scores are unique and reliable. Therefore, below we provide some general and useful properties about of the similarity function in Eq. \ref{eqn:ff}.


\begin{prop}[Global Optimality]\label{prop:go}
	Let us define a new matrix 
	\begin{align}\label{eqn:H}
	\mathbf{H}=\left[\begin{array}{ll}
	\omega_{13}\mathbf{I}_{d_t\times d_t} & -\mathbf{W} \\
	-\mathbf{W}^T & \omega_{24}\mathbf{I}_{d_s\times d_s}
	\end{array}\right],
	\end{align}
	where $\mathbf{I}_{d_t\times d_t}$ and $\mathbf{I}_{d_s\times d_s}$ denote two identity matrices with sizes of $d_t\times d_t$ and $d_s\times d_s$ entries, respectively. Then if $\mathbf{H}$ is positive definite (PD) and $\mathcal{Z}_s, \mathcal{Z}_t$ are nonempty closed convex sets, there exists a unique global solution for Eq. \ref{eqn:ff}.
\end{prop}
\begin{proof}
	Eq. \ref{eqn:ff} can be rewritten with $\mathbf{H}$ in Eq. \ref{eqn:H} as follows:
	\begin{align}\label{eqn:fff}
	& f(x, y; \mathbf{W}, \boldsymbol{\omega}) \nonumber \\
	& = \min_{\mathbf{z}}\left\{\frac{1}{2}\mathbf{z}^T\mathbf{H}\mathbf{z} - \mathbf{z}^Tg(x,y;\boldsymbol{\omega}) + h(x, y, \boldsymbol{\omega})\right\},
	\end{align}
	where $\mathbf{z}=\left[\begin{array}{l}
	\mathbf{z}_t\\
	\mathbf{z}_s
	\end{array}\right]$, $g(x,y;\boldsymbol{\omega})=\left[\begin{array}{l}
	\omega_1\phi(x) \\ \omega_2\psi(y)
	\end{array}\right]$,  
	and $h(x, y, \boldsymbol{\omega})=\frac{\omega_1}{2}\|\phi(x)\|_2^2 + \frac{\omega_2}{2}\|\psi(y)\|_2^2$. Since $\mathcal{Z}_s, \mathcal{Z}_t$ are nonempty closed convex sets, the feasible domain for $\mathbf{z}$ is nonempty closed convex as well. Based on \cite{opac-b1108032}, we can prove this property.
\end{proof}

\begin{prop}[Global Convergence of AO]\label{prop:global}
	Under the conditions in Property \ref{prop:go}, the alternating optimization in Eq. \ref{eqn:zt} and Eq. \ref{eqn:zs} can guarantee global convergence.
\end{prop}
\begin{proof}
	Due to matrix $\mathbf{H}$ being PD, we can have $\omega_{13}>0$ and $\omega_{24}>0$. Further since $\mathcal{Z}_s, \mathcal{Z}_t$ are nonempty closed convex sets, both Eq. \ref{eqn:zt} and Eq. \ref{eqn:zs} define convex optimization problems (see \cite{opac-b1108032}), respectively. Now based on Property \ref{prop:go} we can prove this property.
\end{proof}

\begin{prop}[Local Convergence of AO]\label{prop:local}
	If $\omega_{13}>0$, $\omega_{24}>0$, and $\mathcal{Z}_s, \mathcal{Z}_t$ are nonempty closed convex sets, then the alternating optimization in Eq. \ref{eqn:zt} and Eq. \ref{eqn:zs} can guarantee to converge to local optima.
\end{prop}
\begin{proof}
	Please refer to the proof for Property \ref{prop:global}.
\end{proof}

\begin{prop}[Extreme Case]\label{prop:2}
	Suppose that all the vectors and matrix in Eq.~\ref{eqn:ff} are upper-bounded. Then we have
	\begin{align}
	\lim_{\omega_1, \omega_2\rightarrow+\infty, \omega_3, \omega_4\rightarrow 0} f(x, y; \mathbf{W}, \boldsymbol{\omega}) = \phi(x)^T\mathbf{W}\psi(y).
	\end{align}
\end{prop}
From Property \ref{prop:2} we can easily see that our adaptive similarity function in Eq. \ref{eqn:ff} can be taken as the generalization of the bilinear compatibility function defined in \cite{akata2013label, Akata2015}, and so does our learning framework in Eq. \ref{eqn:general_learning} accordingly.\\

\subsection{A Specific Learning Algorithm}
With various feasible domains $\mathcal{Z}_s, \mathcal{Z}_t$, we can design different adaptive 
similarity functions accordingly. Particularly here we define 
\begin{align}
& \mathcal{Z}_s = \{\mathbf{z}_s \mid \|\mathbf{z}_s\|_2^2\leq\gamma_s, \gamma_s\geq0, \forall\mathbf{z}_s\}, \label{eqn:Zs}\\
& \mathcal{Z}_t = \{\mathbf{z}_t \mid \|\mathbf{z}_t\|_2^2\leq\gamma_t, \gamma_t\geq0, \forall\mathbf{z}_t\}. \label{eqn:Zt}
\end{align}
That is, we define $\mathcal{Z}_s, \mathcal{Z}_t$ to be sufficiently large sets which contain {\em any possible} source or target adapted feature embedding, respectively\footnote{Intuitively we can set $\gamma_s, \gamma_t\rightarrow +\infty$, \ie very large real numbers.}. Our reasoning for this choice is its simplicity and our need for high computational efficiency.

Then by setting the first derivative of $f$ over $\mathbf{z}$ to 0, \ie $\frac{\partial f}{\partial \mathbf{z}}=0$, we can easily get the close-form solution for $\mathbf{z}$, equivalently for $\mathbf{z}_s, \mathbf{z}_t$, as follows:  
\begin{align}\label{eqn:z}
\mathbf{z} = \mathbf{H}^{\dagger}g(x,y;\boldsymbol{\omega}),
\end{align}
where $\dagger$ denotes the pseudo-inverse operation.

\begin{figure}[t]
	\begin{minipage}[b]{0.325\columnwidth}
		\begin{center}
			\centerline{\includegraphics[width=1.1\columnwidth]{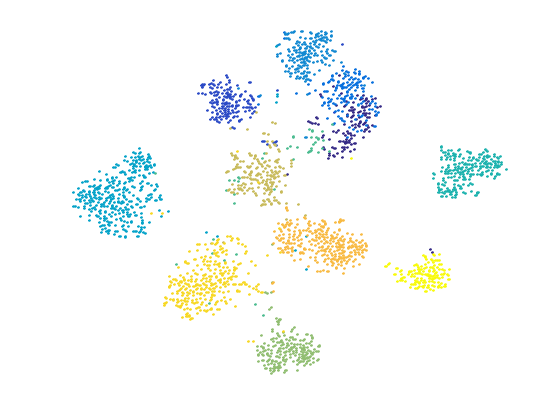}}
			\centerline{\footnotesize{(a) Original features}}
		\end{center}
	\end{minipage}
	\begin{minipage}[b]{0.325\columnwidth}
		\begin{center}
			\centerline{\includegraphics[width=1.1\columnwidth]{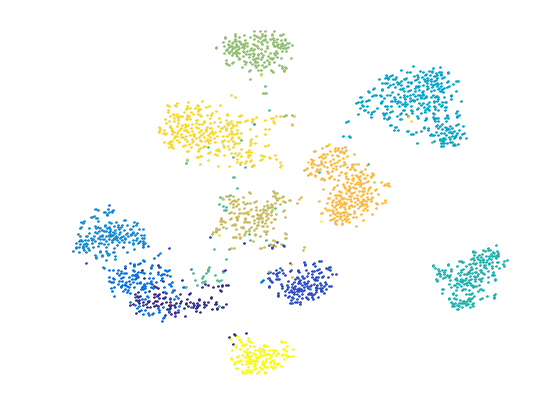}}
			\centerline{\footnotesize{(b) Adapted features $\mathbf{z}_t$}}
		\end{center}
	\end{minipage}
	\begin{minipage}[b]{0.325\columnwidth}
		\begin{center}
			\centerline{\includegraphics[width=1.1\columnwidth]{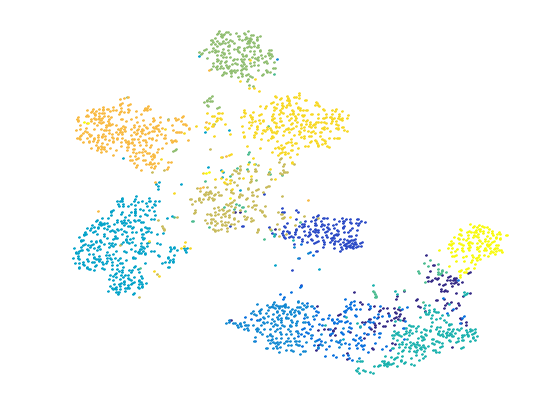}}
			\centerline{\footnotesize{(c) Adapted features $\mathbf{z}_s$}}
		\end{center}
	\end{minipage}	
	\caption{\footnotesize{t-SNE visualization on aP\&Y between {\bf (a)} target-domain original features, {\bf (b)} target-domain adapted features $\mathbf{z}_t$ when matching with an unseen class, and {\bf (c)} corresponding source-domain adapted features $\mathbf{z}_s$. Here each color represents a unique unseen class across the sub-figures.
	}}\label{fig:data}
\end{figure}

\noindent
{\bf Discussion:}
Eq. \ref{eqn:z} suggests a linear transform function of combining source and target information to generate the adapted features. Since $\mathbf{H}$ is PD (and thus so is $\mathbf{H}^{\dagger}$), the target domain data structures are fully preserved in $\mathbf{z}_t$ while matching with a single source domain vector. Correspondingly the target data structures will have a larger impact in generating $\mathbf{z}_s$ as well. Fig. \ref{fig:data} illustrates the distributions of different features using the test data in aP\&Y dataset, which conform with our analysis.

Next we substitute Eq. \ref{eqn:ff}, \ref{eqn:Zs} and \ref{eqn:Zt} into Eq.~\ref{eqn:general_learning} to learn the parameters in $f$. Note that in order to achieve global optimality in Property \ref{prop:go}, the learned parameters must guarantee that matrix $\mathbf{H}$ in Eq. \ref{eqn:H} is PD. This leads us to the following learning problem:
\begin{align}\label{eqn:hard-to-solve}
\min_{\mathbf{W}, \boldsymbol{\omega}, \boldsymbol{\xi}} & \; \frac{\lambda_1}{2}\|\mathbf{W}\|_F^2 + \frac{\lambda_2}{2}\|\boldsymbol{\omega}\|_2^2 + \sum_i \xi_i \\ 
\mbox{s.t.} \hspace{2mm} & \; f(x_i, y_i; \mathbf{W}, \boldsymbol{\omega}) - f(x_i, y; \mathbf{W}, \boldsymbol{\omega}) \geq \Delta(y_i, y) - \xi_i, \nonumber\\
& \; \mathbf{H}(\mathbf{W}, \boldsymbol{\omega})\succ\mathbf{0}, \nonumber\\
& \; \forall i, \xi_i\geq0,  x_i\in\mathcal{X}, y_i, y\in\mathcal{L}_o, \nonumber
\end{align}
where $\mathbf{H}(\mathbf{W}, \boldsymbol{\omega}) \equiv\mathbf{H}$ in Eq. \ref{eqn:H} with $\mathbf{W}, \boldsymbol{\omega}$ as parameters, ``$\succ\mathbf{0}$'' denotes the PD constraint which makes it very difficult to solve the problem, and $\lambda_1\geq0, \lambda_2\geq0$ are two predefined regularization parameters.

As a relaxation we tried to solve Eq. \ref{eqn:hard-to-solve} without considering the PD constraint. However, we observed empirically that the learned parameters do not always satisfy the PD constraint using AO procedure. This leads to poor recognition performance. On the other hand, if we assume that the maximum $\ell_1$ norm of the row vectors $\mathbf{W}_{i,\cdot}, \forall i$ or column vectors $\mathbf{W}_{\cdot,j}, \forall j$ in matrix $\mathbf{W}$, denoted by
\begin{align}\label{eqn:delta}
\delta_W = \max\left\{\max_i\|\mathbf{W}_{i,\cdot}\|_1, \max_j\|\mathbf{W}_{\cdot,j}\|_1\right\},
\end{align}
is non-zero and upper-bounded (which is always the case), we can obtain global optimality at least by manually setting parameter $\boldsymbol{\omega}$ so that $\omega_{13}\geq\delta_W$ and $\omega_{24}\geq\delta_W$. This creates a diagonally dominant matrix for $\mathbf{H}$ and guarantees that PD is satisfied. 

Based on this consideration, we chose not to learn parameter $\boldsymbol\omega$ but instead set it manually to guarantee the PD constraint during training. We thus only learn parameter $\mathbf{W}$. Our learning formulation can now be rewritten as follows:
\begin{align}\label{eqn:learning1}
\min_{\mathbf{W}, \boldsymbol{\xi}} & \; \frac{\lambda}{2}\|\mathbf{W}\|_F^2 + \sum_i \xi_i \\ 
\mbox{s.t.} & \; f(x_i, y_i; \mathbf{W}, \boldsymbol{\omega}^*) - f(x_i, y; \mathbf{W}, \boldsymbol{\omega}^*) \geq \Delta(y_i, y) - \xi_i, \nonumber\\
& \; \forall i, \xi_i\geq0,  x_i\in\mathcal{X}, y_i, y\in\mathcal{L}_o, \nonumber
\end{align}
where $\boldsymbol\omega^*$ denotes the predefined parameter vector and $\lambda\geq0$ is a predefined constant. Since in our experiments the current ZSR problem is essentially equivalent to a multi-class prediction problem, we simply set $\Delta(y_i, y)=1$ if $y_i\neq y$, otherwise 0. 

%
%

In test time, by substituting Eq. \ref{eqn:ff}, \ref{eqn:H} and \ref{eqn:z} into Eq. \ref{eqn:y**} we can rewrite the our decision function for ZSR as follows:
\begin{align}\label{eqn:yy}
& \bar{y}^{*}  = \argmax_{\bar{y}\in\mathcal{L}_u} F(\bar{x}, \bar{y}; \mathbf{H}^{\dagger}, \boldsymbol\omega^*) \\
& = \argmax_{\bar{y}\in\mathcal{L}_u} \left\{\frac{1}{2}g(\bar{x},\bar{y};\boldsymbol{\omega}^*)^T\mathbf{H}^{\dagger}g(\bar{x},\bar{y};\boldsymbol{\omega}^*) - h(\bar{x}, \bar{y}, \boldsymbol{\omega}^*)\right\}. \nonumber
\end{align}

\noindent
{\bf Discussion:} 
Learning based on Eq. \ref{eqn:learning1} has convergence issues due to the nature of latent structural SVMs. On the other hand, since our decisions are based on $\mathbf{H}^{\dagger}$ explicitly as in Eq.~\ref{eqn:yy}, it might be possible to learn $\mathbf{H}^{\dagger}$ approximately and efficiently by substituting similarity function $F$ in Eq.~\ref{eqn:yy} into structural SVMs \cite{joachims2009cutting}. It turns out that this learning strategy is equivalent to \cite{akata2013label,Akata2015} with source-domain feature augmentation, and thus leads to global convergence (under the multi-class prediction setting for ZSR). Empirically we tested this learning strategy and found marginal differences from \cite{akata2013label,Akata2015} in terms of recognition performance. Therefore we do not report these results in our experimental section.

\section{Experiments}
\begin{table}[t]\footnotesize
	\centering
	\setlength\tabcolsep{3pt}
	\caption{\footnotesize{Statistics of different benchmark image datasets.}}\label{tab:dataset}
	\vspace{1mm}
	\begin{tabular}{|l|lll|}
		\hline
		Dataset & \# instances & \# attributes & \# seen/unseen cls. \\
		\hline\hline  
		aP\&Y & 15,339 & 64 (continuous) & 20 / 12 \\
		AwA & 30,475 & 85 (continuous) & 40 / 10 \\
		CUB-200-2011 & 11,788 & 312 (binary) & 150 / 50 \\
		SUN Attribute & 14,340 & 102 (binary) & 707 / 10 \\
		\hline
	\end{tabular}
	\vspace{-3mm}
\end{table}

We follow the experimental settings in \cite{zhang2016zero}. Specifically we test our method on four benchmark image datasets for zero-shot recognition, namely, aPascal \& aYahoo (aP\&Y) \cite{farhadi2009attribute}, Animals with Attributes (AwA) \cite{citeulike:7491128}, Caltech-UCSD Birds-200-2011 (CUB-200-2011) \cite{WahCUB_200_2011}, and SUN Attribute \cite{patterson2014sun}. We summarize the statistics in each dataset and list them in Table~\ref{tab:dataset}. 

For aP\&Y, CUB-200-2011 and SUN Attribute datasets, we take the means of attribute vectors from the same classes to generate source domain data. For AwA dataset, we utilize the real-number attribute vectors since they are more discriminative. For all the datasets, we utilize MatConvNet \cite{arXiv:1412.4564} with the ``imagenet-vgg-verydeep-19'' pretrained model \cite{simonyan2014very} to extract a 4096-dim CNN feature vector (\ie the top layer hidden unit activations of the network) for each image (or bounding box). As suggested in \cite{zhang2016zero} and with the same parameters we conduct dimension reduction for target-domain data and sparse coding for source-domain attribute vectors as well. All the predefined parameters in our method are tuned using cross-validation, similar to \cite{Zhang2015,zhang2016zero}. We report our results averaged over 10 trials.

We utilize the same standard training/testing splits for zero-shot recognition on aP\&Y and AwA as others, defined in these datasets. For CUB-200-2011, we follow \cite{akata2013label} and use the same 150 bird species as seen classes for training, and the other 50 species as unseen classes for testing. For SUN Attribute, we follow \cite{jayaraman2014unreliable} and use the same 10 classes as unseen classes for testing (see their supplementary file), and the rest of them as seen classes for training.

On the four datasets, we evaluate the performance of the proposed approach in terms of recognition under the following two different settings:



\noindent
{\bf Standard \vs Transductive}. Our goal here is to benchmark the performance gains with different types of side-information in classification. Standard setting represents an extreme streaming situation. 
Transductive setting provides a less extreme scenario where during test-time we are given target instances all at once as in a batch-mode. The batch mode clearly provides information about the target data/feature distribution. The question arises as to how much we could benefit from this type of information. In this context we propose taking the similarities from our approach as inputs to the method proposed in \cite{zhang2016ECCV} for recognition. Here we report our results averaged over 10 trials. In each trial we run our integrated approach for another 50 times and record the average as probabilities over unseen classes per target data. We predict class labels and report our performance based on this assignment probability matrix in each trial.


\subsection{Parameter Selection}
\begin{figure}[t]
	\begin{minipage}[b]{0.49\linewidth}
		\begin{center}
			\centerline{\includegraphics[width=.95\linewidth]{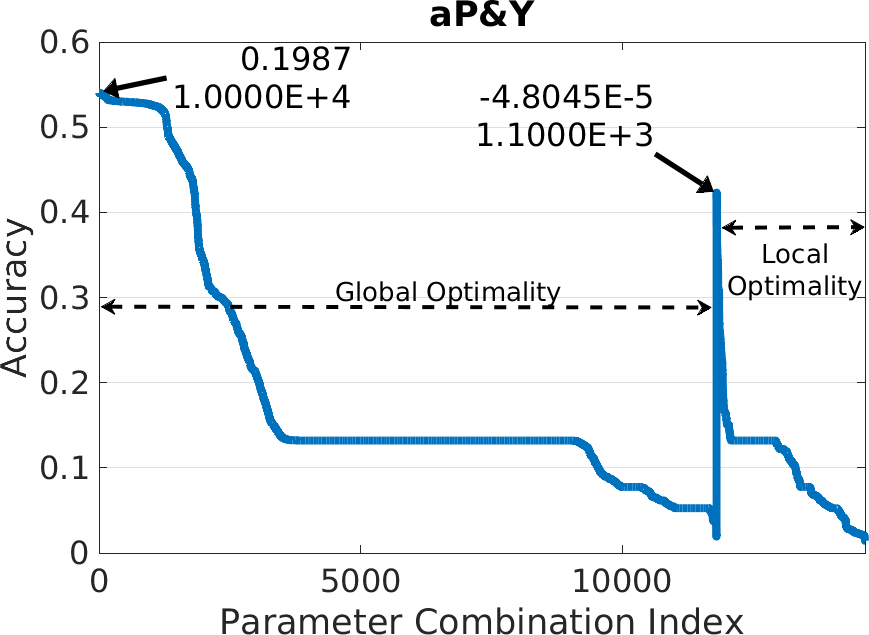}}
		\end{center}
	\end{minipage}
	\begin{minipage}[b]{0.49\linewidth}
		\begin{center}
			\centerline{\includegraphics[width=.95\linewidth]{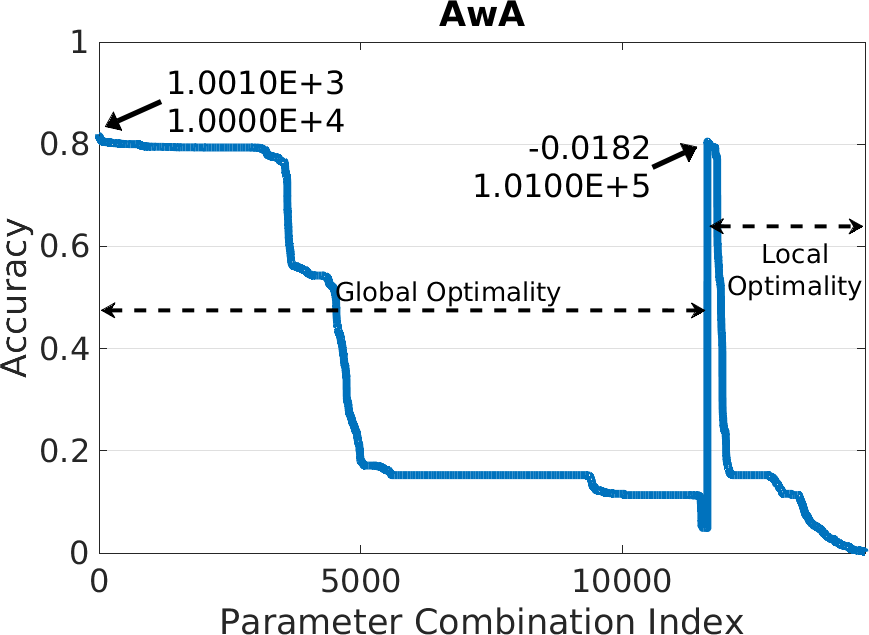}}
		\end{center}
	\end{minipage}	
	\begin{minipage}[b]{0.49\linewidth}
		\begin{center}
			\centerline{\includegraphics[width=.95\linewidth]{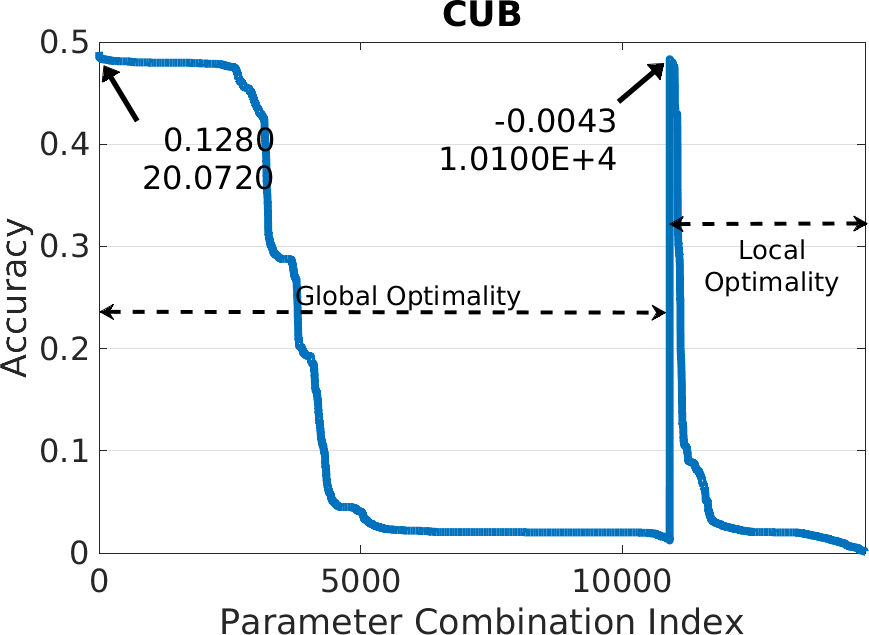}}
		\end{center}
	\end{minipage}
	\begin{minipage}[b]{0.49\linewidth}
		\begin{center}
			\centerline{\includegraphics[width=.95\linewidth]{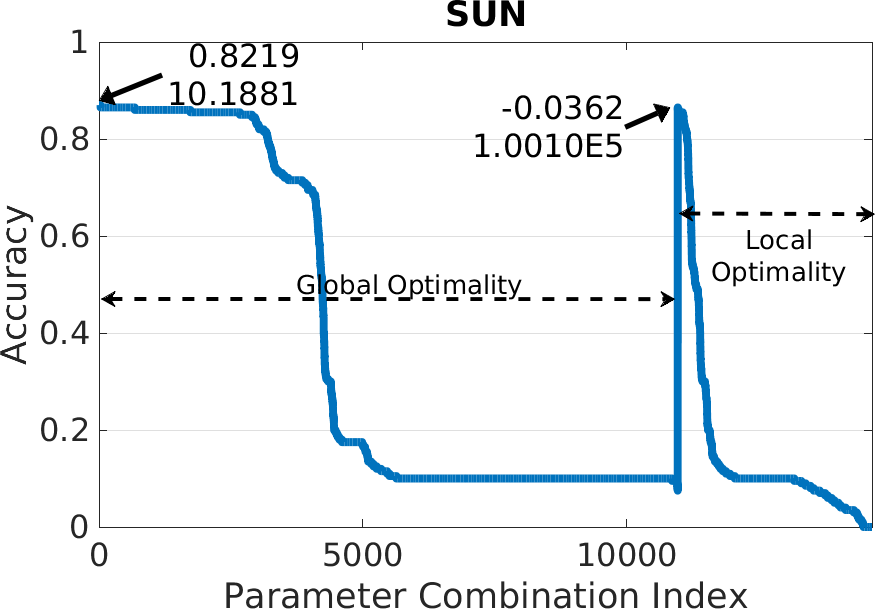}}
		\end{center}
	\end{minipage}	
	\caption{\footnotesize{{\bf Global optimality \vs local optimality:} Performance comparison over different parameter combinations of $\boldsymbol{\omega}$. For both cases the reported results here have been sorted in a descending order. The numbers on the curves are the associated (top) smallest and (bottom) largest eigenvalues of matrix $\mathbf{H}$ with the best results.
	}}\label{fig:PD}
	\vspace{-3mm}
\end{figure}

In our method there are 5 predefined parameters, \ie $\lambda\geq0$ and $\boldsymbol{\omega}=[\omega_1; \omega_2; \omega_3; \omega_4]$. In this section, we will investigate the impact of $\boldsymbol{\omega}$ on recognition accuracy while fixing $\lambda=1$ without fine-tuning. Specifically we conduct a grid search over $10^{-5:5}$ for each parameter in $\boldsymbol{\omega}$, \ie 11 choices per parameter and $11^4=14,641$ parameter combinations in total. In this way our adaptive similarity function in Eq. \ref{eqn:ff} always converges to a local maxima (\ie local optimality) because of Property \ref{prop:local}. Further when the corresponding matrix $\mathbf{H}$ in Eq. \ref{eqn:H} is PD, our function achieves global maxima (\ie global optimality) because of Property~\ref{prop:global}. We utilize the smallest eigenvalue of $\mathbf{H}$ as an indicator of being PD if it is positive, otherwise not.

\begin{figure}[t]
	\begin{minipage}[b]{\linewidth}
		\begin{center}
			\centerline{\includegraphics[width=\linewidth]{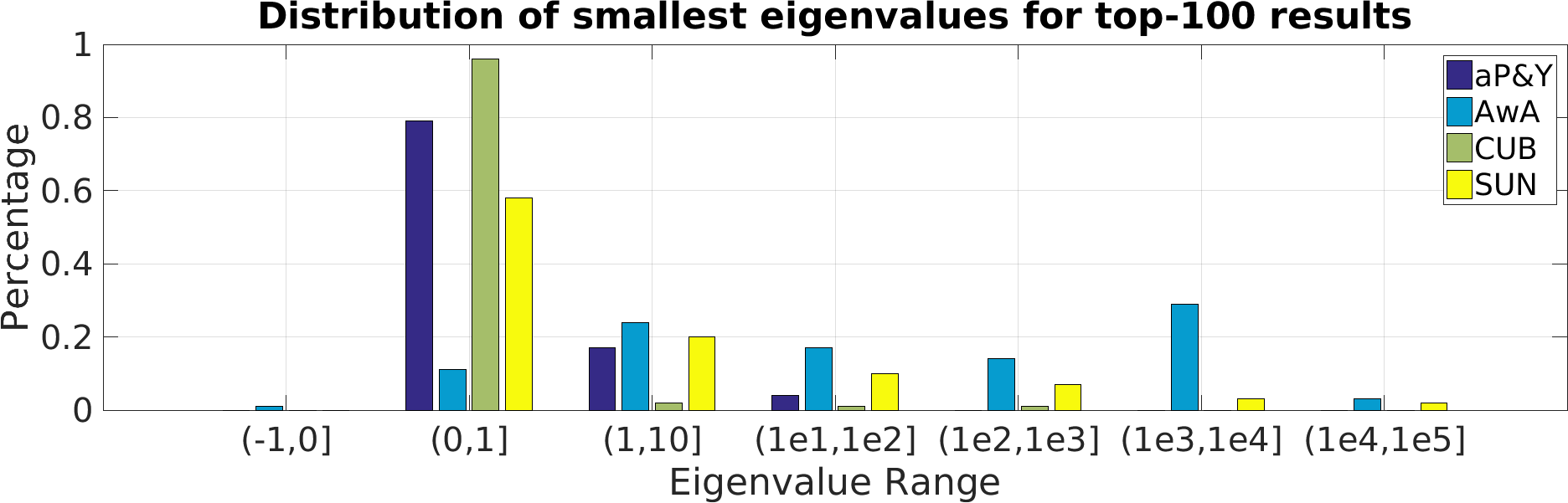}}
		\end{center}
	\end{minipage}
	\vfill
	\begin{minipage}[b]{\linewidth}
		\begin{center}
			\centerline{\includegraphics[width=\linewidth]{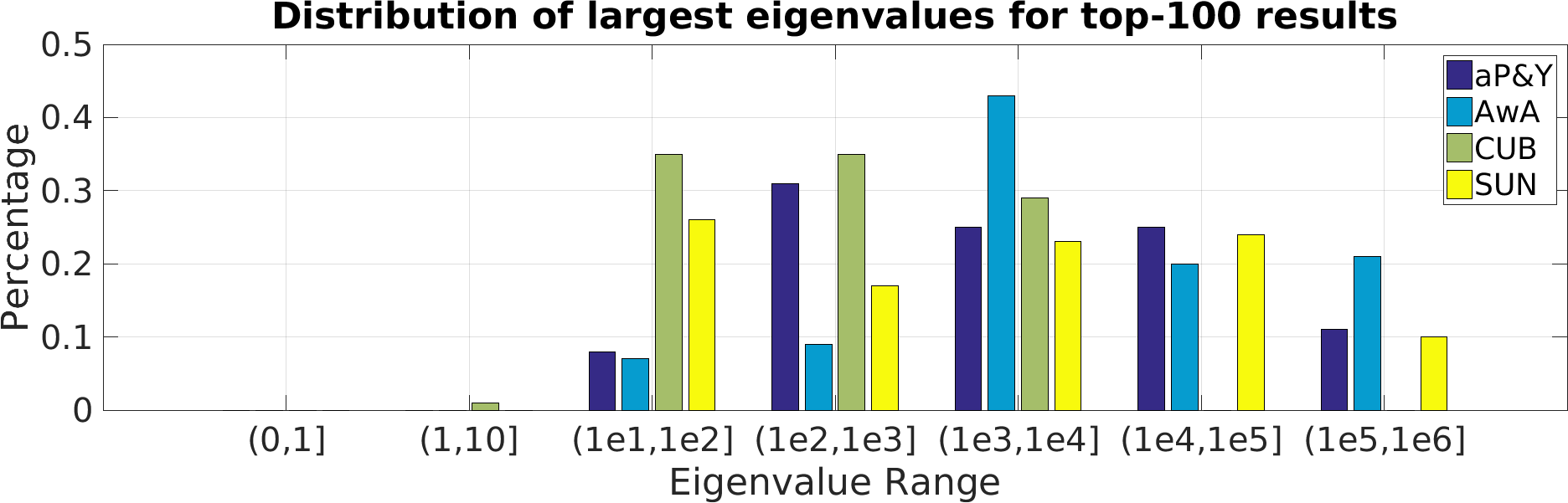}}
		\end{center}
	\end{minipage}	
	\vspace{-7mm}
	\caption{\footnotesize{Illustration of distributions of {\bf (top)} smallest and {\bf (bottom)} largest eigenvalues for the corresponding top-100 results on each dataset.
	}}
	\label{fig:param}
	\vspace{-3mm}
\end{figure}

First we investigate the effect of global/local optimality on test-time recognition accuracy, and illustrate the results in Fig.~\ref{fig:PD}. Clearly we can observe that (1) {\em Quality:} The best global optimality results outperform their local optimality counterparts. (2) {\em Robustness:} The highest performing global optimality solution over different parameter combinations is more robust. This is indicated by the fact that for global optimality the curves for the top performance are going down very slowly, forming wide and relatively flat ranges, while for local optimality the curves are going down rapidly. The robustness here indeed suggests that with global optimality the parameter selection for $\boldsymbol{\omega}$ could have a sufficient number of choices, making it relatively easy. (3) {\em Eigenvalues:} In general the smallest eigenvalues tend to be close to 0, but the largest eigenvalues tend to be very large.

In order to accelerate the process of parameter selection, we provide some insights by looking at the distributions of both smallest and largest eigenvalues, as illustrated in Fig.~\ref{fig:param}. Generally speaking, smallest eigenvalues tend to concentrate on the range $(0, 1]$, while largest ones are uniformly distributed between 10 to $10^6$ with slightly better focus on the range $(10^3, 10^4]$. This in turn suggests that for our method a good parameter combination for $\boldsymbol{\omega}$ may lead to a PD matrix $\mathbf{H}$ (with high probability) whose smallest and largest eigenvalues lie in $(0,1]$ and $(10^3, 10^4]$, respectively. By quickly checking this condition, we can easily rule out most potential combinations. Further the big difference between the smallest and largest eigenvalues indicates a big difference between $\omega_{13}$ and $\omega_{24}$ as well. Moreover, if $\delta_W$ in Eq. \ref{eqn:delta} can be estimated, we may select parameter combinations more efficiently by constructing diagonal dominant matrices intentionally. This shows that the special structure of $\mathbf{H}$ may provide important information to guide parameter selection.

\subsection{Zero-Shot Recognition}
In this section we compare our method, denoted by JFA, with other existing ZSL/ZSR approaches. This task is fundamentally about classification when a single target data instance is presented in test time. 

\begin{figure*}[t]
	\begin{minipage}[b]{0.325\textwidth}
		\begin{center}
			\centerline{\includegraphics[width=\columnwidth]{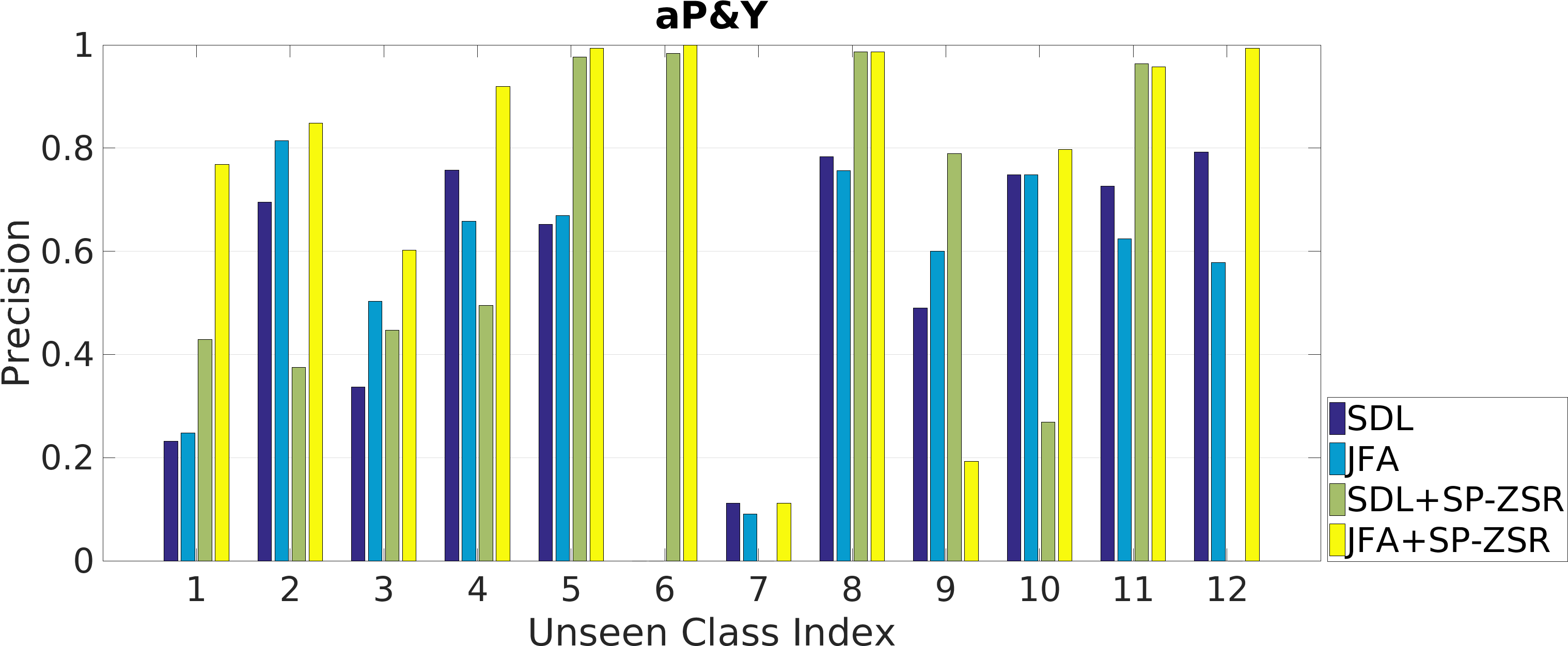}}
		\end{center}
	\end{minipage}
	\begin{minipage}[b]{0.325\textwidth}
		\begin{center}
			\centerline{\includegraphics[width=\columnwidth]{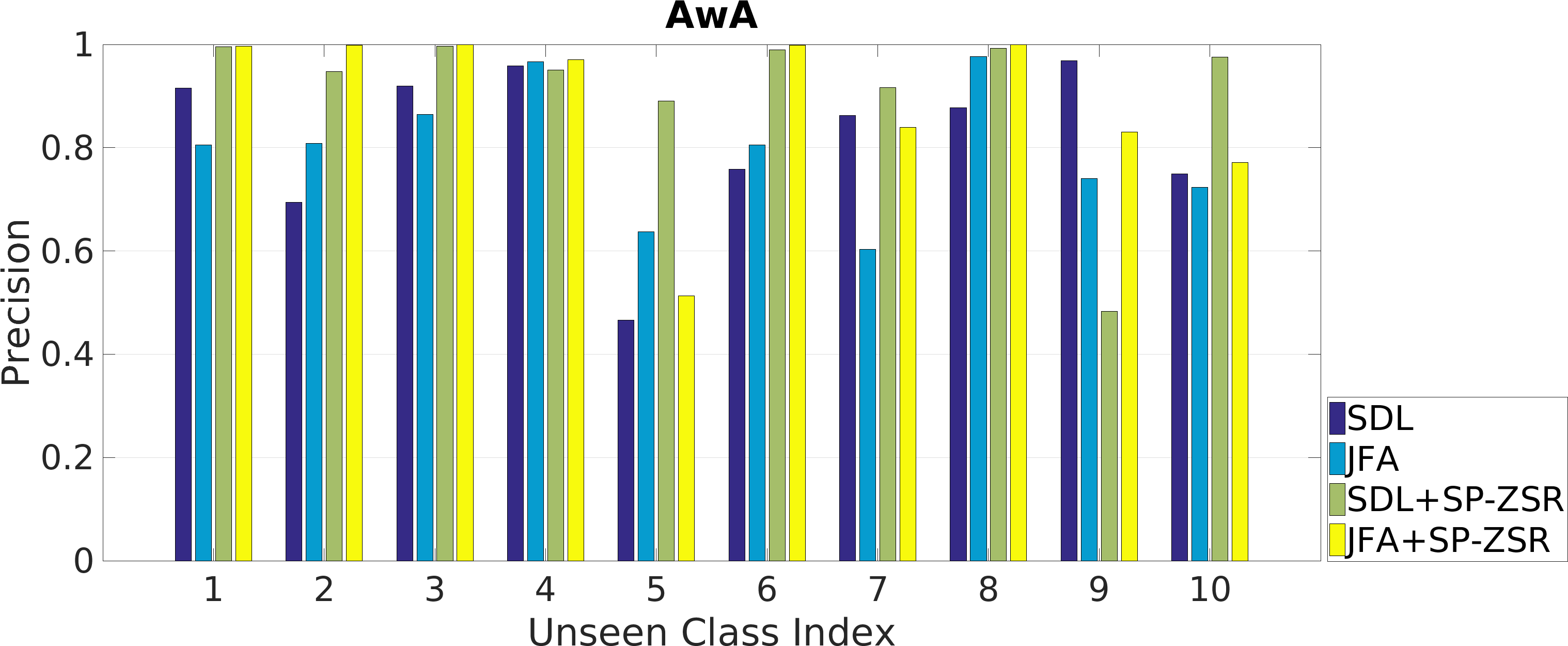}}
		\end{center}
	\end{minipage}
	\begin{minipage}[b]{0.325\textwidth}
		\begin{center}
			\centerline{\includegraphics[width=\columnwidth]{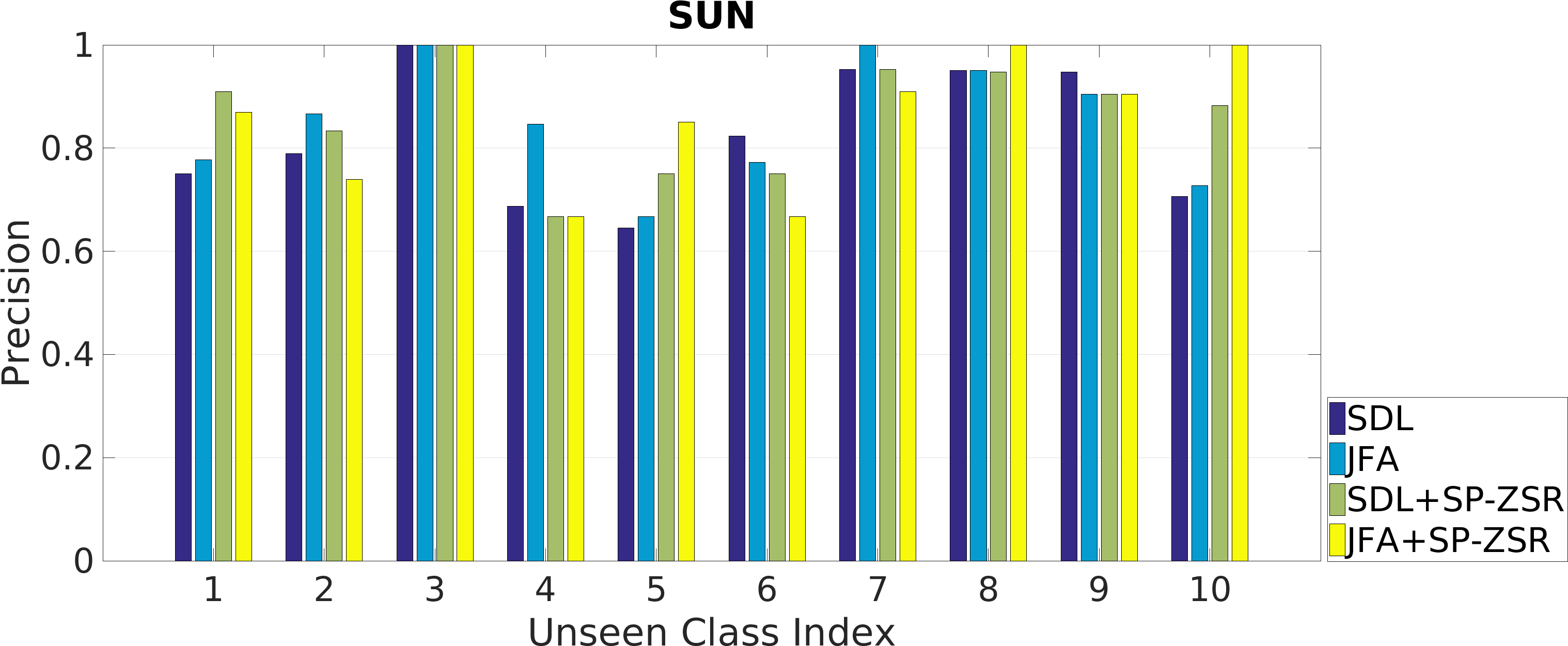}}
		\end{center}
	\end{minipage}
	\begin{minipage}[b]{0.325\textwidth}
		\begin{center}
			\centerline{\includegraphics[width=\columnwidth]{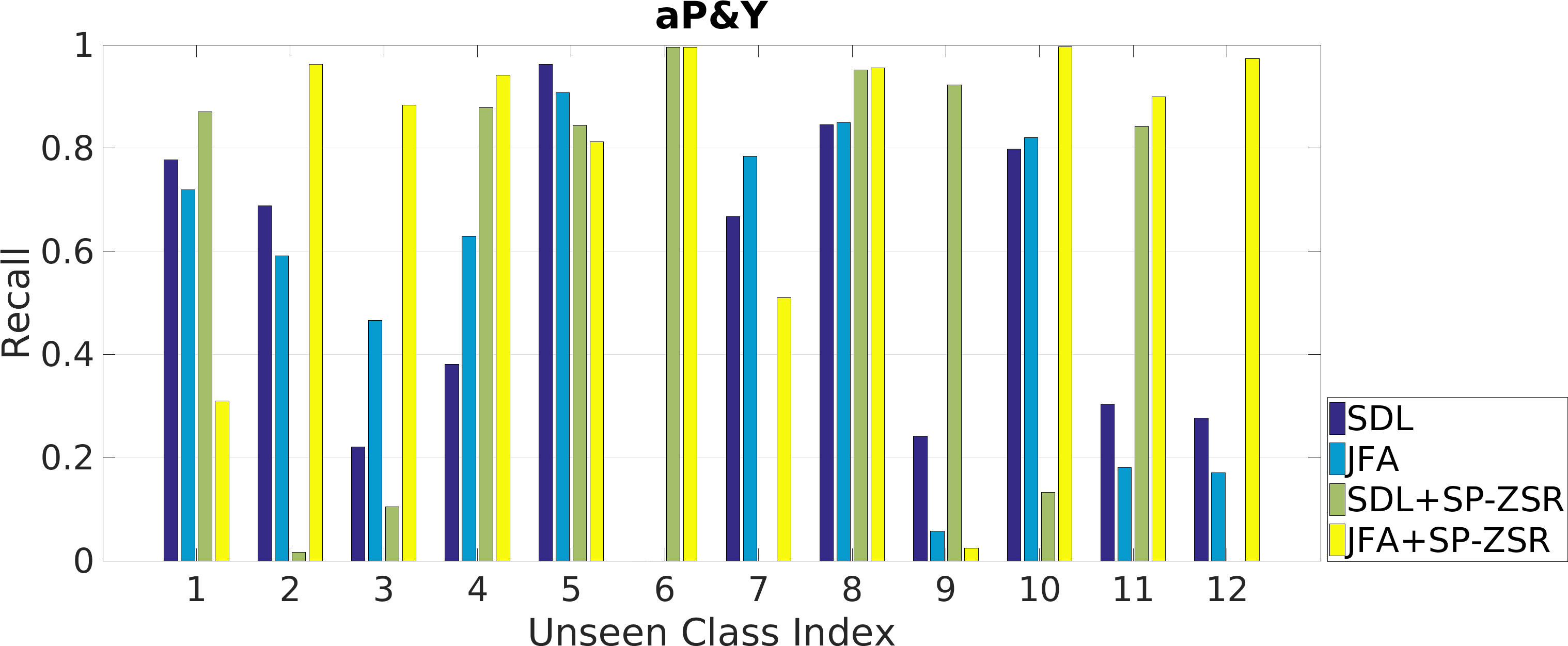}}
		\end{center}
	\end{minipage}
	\begin{minipage}[b]{0.325\textwidth}
		\begin{center}
			\centerline{\includegraphics[width=\columnwidth]{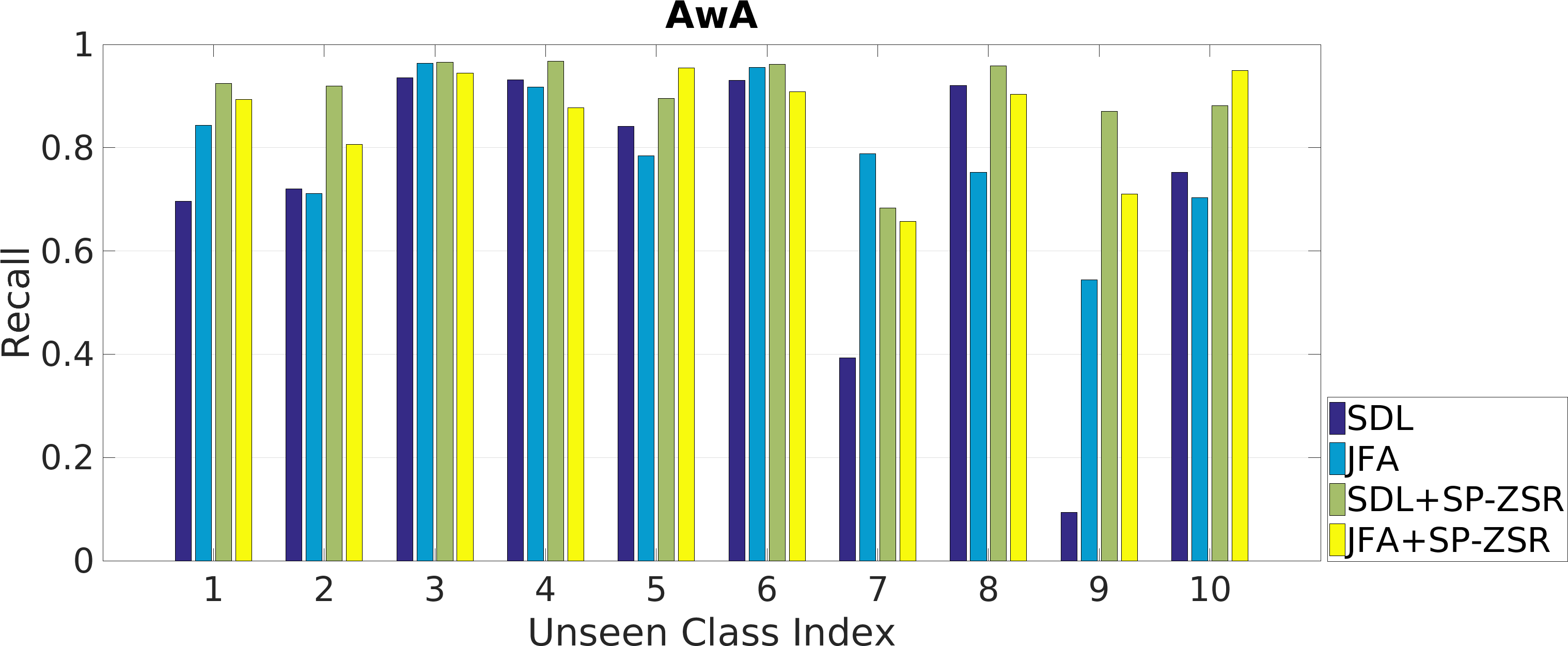}}
		\end{center}
	\end{minipage}
	\begin{minipage}[b]{0.325\textwidth}
		\begin{center}
			\centerline{\includegraphics[width=\columnwidth]{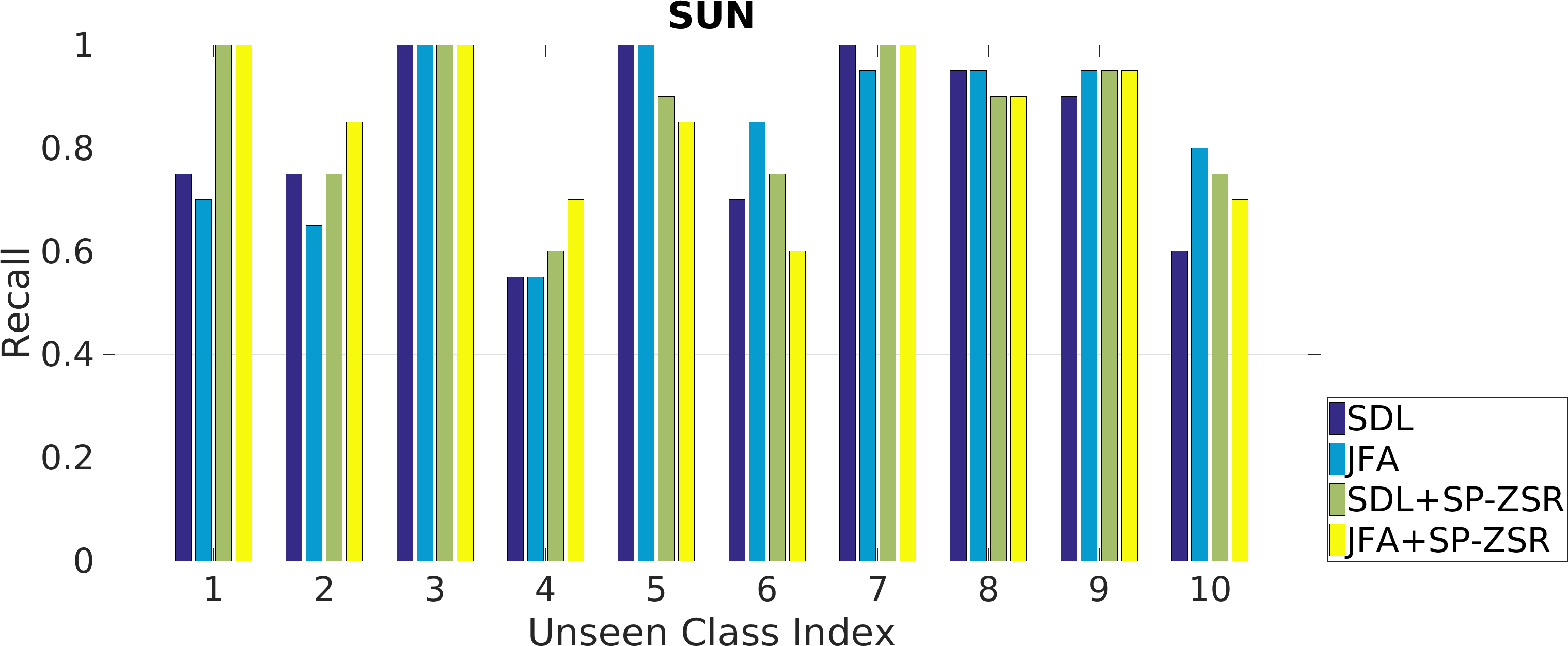}}
		\end{center}
	\end{minipage}
	\vspace{-3mm}
	\caption{\footnotesize{Performance comparison in terms of {\bf (top)} precision and {\bf (bottom)} recall among unseen classes indexed by the orders in the corresponding datasets. The numbers for the competitors are cited from \cite{zhang2016ECCV}.}}\label{fig:recog-pre}
	\vspace{-3mm}
\end{figure*}

\subsubsection{Standard Setting}
\begin{table}[t]\scriptsize
	\begin{minipage}{\linewidth}
		\centering
		\caption{\footnotesize{Traditional zero-shot recognition accuracy comparison (\%) with cited numbers in the form of mean$\pm$standard deviation, grouped by the image feature types of {\bf (top)} handcrafted features, {\bf (middle)} other deep learning features, and {\bf (bottom)} vgg-verydeep-19. Blanks indicate no reports for the datasets in the original papers.}}\label{tab:standard_recognition}
		\vspace{1mm}
		\setlength\tabcolsep{3pt}
		\begin{tabular}{|lllll|}
			\hline
			Method & aP\&Y & AwA & CUB & SUN \\
			\hline\hline
			Farhadi \etal \cite{farhadi2009attribute} & 32.5 & & & \\
			Mahajan \etal \cite{mahajan2011joint} & 37.93 & & & \\
			Wang and Ji \cite{wang2013unified} & 45.05 & 42.78 & & \\
			Rohrbach \etal \cite{conf/nips/RohrbachES13} & & 42.7 & & \\
			Yu \etal \cite{yu2013designing} & & 48.30 & & \\
			Akata \etal \cite{akata2013label} & & 43.5 & 18.0 & \\
			Mensink \etal \cite{mensink2014costa} & & & 14.4 & \\
			Lampert \etal \cite{10.1109/TPAMI.2013.140} & 19.1 & 40.5 & & 52.50 \\
			J. and Grauman \cite{jayaraman2014unreliable} & 26.02$\pm$0.05 & 43.01$\pm$0.07 & & 56.18$\pm$0.27 \\
			R.-P. and Torr \cite{Romera-Paredes2015} & 27.27$\pm$1.62 & 49.30$\pm$0.21 & & 65.75$\pm$0.51 \\
			\hline\hline
			Akata \etal \cite{Akata2015} & & 66.7 & \textbf{\em 50.1} & \\
			Qiao \etal \cite{Qiao_2016_CVPR} & & 66.46$\pm$0.42 & 29.00$\pm$0.28 & \\
			Changpinyo \etal \cite{Changpinyo_2016_CVPR} & & 72.9 & & \\		
			Xian \etal \cite{Xian2016Latent} & & 71.9 & 45.5 &  \\
			Wang \etal \cite{wang2016relational} & & 75.99 & 33.48 & \\
			\hline\hline
			Lampert \etal \cite{10.1109/TPAMI.2013.140} & 38.16 & 57.23 & & 72.00 \\
			R.-P. and Torr \cite{Romera-Paredes2015} & 24.22$\pm$2.89 & 75.32$\pm$2.28 & & 82.10$\pm$0.32 \\
			SSE-INT \cite{Zhang2015} & 44.15$\pm$0.34 & 71.52$\pm$0.79 & 30.19$\pm$0.59 & 82.17$\pm$0.76\\
			SSE-ReLU \cite{Zhang2015} & 46.23$\pm$0.53 & 76.33$\pm$0.83 & 30.41$\pm$0.20 & 82.50$\pm$1.32\\
			SDL \cite{zhang2016zero} & 50.35$\pm$2.97 & 79.12$\pm$0.53 & 41.78$\pm$0.52 & 83.83$\pm$0.29 \\			
			Bucher \etal \cite{bucher2016improving} & \textbf{\em 53.15$\pm$0.88} & 77.32$\pm$1.03 & 43.29$\pm$0.38 & \textbf{\em 84.41$\pm$0.71} \\
			{\bf Ours: JFA} & 52.04$\pm$1.35 & \textbf{\em 81.03$\pm$0.88} & 46.48$\pm$1.67 & 84.10$\pm$1.51 \\
			\hline
		\end{tabular}
	\end{minipage}
\end{table}

Table \ref{tab:standard_recognition} summarizes our comparison under the standard ZSR setting. Clearly ZSL methods leverage the advantages of deep learning features and achieve much better performance than those using handcrafted features. Different deep learning features can achieve comparable performance. Among all the competitors using vgg-verydeep-19 features, our method works the best, outperforming the state-of-the-art \cite{bucher2016improving} by 1.34\% on average. Notice that our experimental setting is exactly the same as \cite{zhang2016zero}, and in this case our method outperforms \cite{zhang2016zero} significantly by 2.11\%, but our standard deviations are slightly higher than \cite{zhang2016zero}. The improvement comes from the nature of adaptive matching with better similarities, while the downside is mainly because our learning algorithm in Eq.~\ref{eqn:learning1} does not converge globally, leading to different local solutions even given the same training data. 

To see this, let us take the CUB dataset for example. Initially CUB is created for fine-grained classification problems with the help of attributes, because some bird species look visually very similar but still have their own unique characteristics. This leads to a more descriptive attribute vector per image than the average. In this context our adaptive matching tries to estimate the individual attribute vector from the average attribute vector of the class for matching based on the visual information. As we see on CUB, our method improves \cite{zhang2016zero} by 4.70\% in terms of accuracy, equivalently 11.25\% relative improvement.

\subsubsection{Transductive Setting}
\begin{table}[t]\scriptsize
	\begin{minipage}{\linewidth}
		\centering
		\caption{\footnotesize{Transductive zero-shot recognition accuracy comparison (\%) with cited numbers in the form of mean$\pm$standard deviation.}}\label{tab:transductive_recognition}
		\vspace{1mm}
		\setlength\tabcolsep{2.3pt}
		\begin{tabular}{|lllll|}
			\hline
			Method & aP\&Y & AwA & CUB & SUN \\
			\hline\hline			
			Fu \etal \cite{embedding2014ECCV} & & 47.1 & & \\			
			Fu \etal \cite{fu2015transductive} &  & 80.5 & 47.9 &  \\
			Kodirov \etal \cite{Kodirov2015} &  & 75.6 & 40.2 & \\
			Guo \etal \cite{guo2016transductive} & 39.03 & 78.47 & & 82.00 \\
			R.\&T. \cite{Romera-Paredes2015}+SP-ZSR \cite{zhang2016ECCV} & 37.5 &
			84.3 & & 89.5 \\
			SDL \cite{zhang2016zero}+SP-ZSR \cite{zhang2016ECCV} &  62.19$\pm$4.65 & \textbf{\em 92.08$\pm$0.14} & 55.34$\pm$0.77 & \textbf{\em 86.12$\pm$0.99} \\
			(BL-ZSL+SP-ZSR) \cite{zhang2016ECCV} & 69.74$\pm$3.47 & 92.06$\pm$0.18 & 53.26$\pm$1.04 & 86.01$\pm$1.32 \\
			{\bf Ours: JFA+SP-ZSR \cite{zhang2016ECCV}} & \textbf{\em 80.89$\pm$5.97} & 88.04$\pm$0.69 & \textbf{\em  55.81$\pm$1.37} & 85.35$\pm$1.56 \\
			\hline
		\end{tabular}
	\end{minipage}
\end{table}

\begin{table}[t]\scriptsize
	\begin{minipage}{\linewidth}
		\centering
		\setlength\tabcolsep{1.8pt}
		\caption{\footnotesize{Average precision and recall comparison (\%) for recognition.}}\label{tab:pre-rec}
		\vspace{1mm}
		\begin{tabular}{|lllll|}
			\hline
			Precision & aP\&Y & AwA & CUB & SUN \\
			\hline
			SDL \cite{zhang2016zero} & 52.70$\pm$27.33 & 81.70$\pm$14.67 & 54.06$\pm$24.13 & 82.51$\pm$12.24 \\
			{\bf Ours: JFA} & 52.41$\pm$25.59 & 79.31$\pm$11.70 & 53.73$\pm$23.90 & 85.12$\pm$10.87 \\
			SDL \cite{zhang2016zero}+SP-ZSR \cite{zhang2016ECCV} & 55.96$\pm$35.72 & \textbf{\em 91.37$\pm$14.75} & 57.09$\pm$27.91 & 85.96$\pm$10.15 \\						
			{\bf Ours: JFA+SP-ZSR\cite{zhang2016ECCV}} & \textbf{\em 76.41$\pm$29.70} & 89.19$\pm$15.09 & \textbf{\em 57.20$\pm$25.96} & \textbf{\em 86.06$\pm$12.36} \\			
			\hline\hline
			Recall & & & & \\
			\hline
			SDL \cite{zhang2016zero} & 51.34$\pm$29.69 & 72.14$\pm$26.29 & 45.05$\pm$26.16 & 82.00$\pm$16.31 \\
			{\bf Ours: JFA} & 51.48$\pm$31.61 & 79.63$\pm$12.34 & 46.98$\pm$29.81 & 84.00$\pm$15.13 \\	
			SDL \cite{zhang2016zero}+SP-ZSR \cite{zhang2016ECCV} & 54.66$\pm$42.27 & \textbf{\em 90.28$\pm$8.08} & 55.73$\pm$31.80 & \textbf{\em 86.00$\pm$13.19} \\								
			{\bf Ours: JFA+SP-ZSR\cite{zhang2016ECCV}} & \textbf{\em 77.20$\pm$30.45} & 86.04$\pm$9.82 & \textbf{\em 55.77$\pm$26.54} & 85.50$\pm$13.68 \\
			\hline
		\end{tabular}
	\end{minipage}
	\vspace{-3mm}
\end{table}

For transductive setting, we list our comparison results in Table \ref{tab:transductive_recognition}. Overall, our method outperforms the state-of-the-art \cite{zhang2016ECCV} significantly by 2.26\% on average. It is worth mentioning that on aP\&Y by substituting our similarity scores in \cite{zhang2016ECCV} we can achieve 80.89\% in terms of accuracy and outperform the state-of-the-art significantly by 11.15\%. Analogous to the results of the traditional setting, we observe that the standard deviations of our results are slightly higher than those of the competitors.

To better compare our results, we further measure the class-level performance on the datasets in terms of precision and recall (equivalent to accuracy per class). The detailed comparisons are illustrated in Fig. \ref{fig:recog-pre} without the CUB dataset due to the space limit. We summarize the averaged performance across different classes on each dataset in Table~\ref{tab:pre-rec}. Overall at the class level our method behaves similar to \cite{zhang2016zero} with the same inputs. However, as we see there exists no single dominant method over all the datasets and uniformly over all classes on each dataset. Better similarity measure does not necessarily lead to better performance under either standard or transductive setting. It could be interesting as future work to see whether we can improve the ZSR performance further by integrating different similarity metrics.

\section{Conclusion}
In this paper we solve the relative sparseness issue of source-domain attribute vectors in ZSR problems. We formulate ZSL as a latent structural SVMs. To account for the rich data variability in target domain, we propose a novel data-dependent adaptive similarity function that adapts to test-time source and target data instances. Our similarity function searches for latent features from both domains by maximizing the latent similarities as well as minimizing the penalties incurred by feature displacements. To parameterize our adaptive similarity function, we propose a family of bilinear based similarity functions with regularized least squares to penalize displacements. We design a specific function with closed-form solutions and propose its corresponding learning algorithm for ZSR. To demonstrate the effectiveness of our proposed method, we test it on four benchmark datasets for ZSR with comprehensive comparison, and show significant improvement over the state-of-the-art under both standard and transductive settings.

\newpage
{\small
	\bibliographystyle{ieee}
	\bibliography{egbib}

\begin{thebibliography}{10}\itemsep=-1pt

\bibitem{akata2013label}
Z.~Akata, F.~Perronnin, Z.~Harchaoui, and C.~Schmid.
\newblock Label-embedding for attribute-based classification.
\newblock In {\em CVPR}, pages 819--826, 2013.

\bibitem{Akata2015}
Z.~Akata, S.~Reed, D.~Walter, H.~Lee, and B.~Schiele.
\newblock Evaluation of output embeddings for fine-grained image
  classification.
\newblock In {\em CVPR}, June 2015.

\bibitem{Al-Halah_2016_CVPR}
Z.~Al-Halah, M.~Tapaswi, and R.~Stiefelhagen.
\newblock Recovering the missing link: Predicting class-attribute associations
  for unsupervised zero-shot learning.
\newblock In {\em CVPR}, June 2016.

\bibitem{antol2014zero}
S.~Antol, C.~L. Zitnick, and D.~Parikh.
\newblock Zero-shot learning via visual abstraction.
\newblock In {\em ECCV}, pages 401--416. Springer, 2014.

\bibitem{Ba2015}
J.~L. Ba, K.~Swersky, S.~Fidler, and R.~Salakhutdinov.
\newblock Predicting deep zero-shot convolutional neural networks using textual
  descriptions.
\newblock {\em arXiv preprint arXiv:1506.00511}, 2015.

\bibitem{Berg:2010:AAD:1886063.1886114}
T.~L. Berg, A.~C. Berg, and J.~Shih.
\newblock Automatic attribute discovery and characterization from noisy web
  data.
\newblock In {\em ECCV}, pages 663--676, 2010.

\bibitem{opac-b1108032}
D.~P. Bertsekas, A.~E. Ozdaglar, and A.~Nedić.
\newblock {\em Convex analysis and optimization}.
\newblock Athena scientific optimization and computation series. Athena
  Scientific, Belmont (Mass.), 2003.

\bibitem{Bhatia15}
K.~Bhatia, H.~Jain, P.~Kar, M.~Varma, and P.~Jain.
\newblock Sparse local embeddings for extreme multi-label classification.
\newblock In {\em NIPS}, 2015.

\bibitem{bucher2016improving}
M.~Bucher, S.~Herbin, and F.~Jurie.
\newblock Improving semantic embedding consistency by metric learning for
  zero-shot classification.
\newblock {\em arXiv preprint arXiv:1607.08085}, 2016.

\bibitem{chang2015semantic}
X.~Chang, Y.~Yang, A.~G. Hauptmann, E.~P. Xing, and Y.-L. Yu.
\newblock Semantic concept discovery for large-scale zero-shot event detection.
\newblock In {\em AAAI}, pages 2234--2240, 2015.

\bibitem{Changpinyo_2016_CVPR}
S.~Changpinyo, W.-L. Chao, B.~Gong, and F.~Sha.
\newblock Synthesized classifiers for zero-shot learning.
\newblock In {\em CVPR}, June 2016.

\bibitem{chao2016empirical}
W.-L. Chao, S.~Changpinyo, B.~Gong, and F.~Sha.
\newblock An empirical study and analysis of generalized zero-shot learning for
  object recognition in the wild.
\newblock {\em arXiv preprint arXiv:1605.04253}, 2016.

\bibitem{elhoseiny2015zero}
M.~Elhoseiny, J.~Liu, H.~Cheng, H.~Sawhney, and A.~Elgammal.
\newblock Zero-shot event detection by multimodal distributional semantic
  embedding of videos.
\newblock {\em arXiv preprint arXiv:1512.00818}, 2015.

\bibitem{farhadi2009attribute}
A.~Farhadi, I.~Endres, D.~Hoiem, and D.~Forsyth.
\newblock Describing objects by their attributes.
\newblock In {\em CVPR}, pages 1778--1785, 2009.

\bibitem{felzenszwalb2010object}
P.~F. Felzenszwalb, R.~B. Girshick, D.~McAllester, and D.~Ramanan.
\newblock Object detection with discriminatively trained part-based models.
\newblock {\em TPAMI}, 32(9):1627--1645, 2010.

\bibitem{frome2013devise}
A.~Frome, G.~S. Corrado, J.~Shlens, S.~Bengio, J.~Dean, M.~A. Ranzato, and
  T.~Mikolov.
\newblock Devise: A deep visual-semantic embedding model.
\newblock In {\em NIPS}, pages 2121--2129, 2013.

\bibitem{embedding2014ECCV}
Y.~Fu, T.~M. Hospedales, T.~Xiang, Z.~Fu, and S.~Gong.
\newblock Transductive multi-view embedding for zero-shot recognition and
  annotation.
\newblock In {\em ECCV}, 2014.

\bibitem{fu2015transductive}
Y.~Fu, T.~M. Hospedales, T.~Xiang, and S.~Gong.
\newblock Transductive multi-view zero-shot learning.
\newblock {\em PAMI}, 37(11):2332--2345, 2015.

\bibitem{fu2015}
Z.~Fu, T.~Xiang, E.~Kodirov, and S.~Gong.
\newblock Zero-shot object recognition by semantic manifold distance.
\newblock In {\em CVPR}, pages 2635--2644, 2015.

\bibitem{gan2015exploring}
C.~Gan, M.~Lin, Y.~Yang, Y.~Zhuang, and A.~G. Hauptmann.
\newblock Exploring semantic inter-class relationships (sir) for zero-shot
  action recognition.
\newblock In {\em AAAI}, pages 3769--3775, 2015.

\bibitem{guo2016transductive}
Y.~Guo, G.~Ding, X.~Jin, and J.~Wang.
\newblock Transductive zero-shot recognition via shared model space learning.
\newblock In {\em AAAI}, 2016.

\bibitem{hariharan2012efficient}
B.~Hariharan, S.~Vishwanathan, and M.~Varma.
\newblock Efficient max-margin multi-label classification with applications to
  zero-shot learning.
\newblock {\em Machine learning}, 88(1-2):127--155, 2012.

\bibitem{jayaraman2014unreliable}
D.~Jayaraman and K.~Grauman.
\newblock Zero-shot recognition with unreliable attributes.
\newblock In {\em NIPS}, pages 3464--3472, 2014.

\bibitem{joachims2009cutting}
T.~Joachims, T.~Finley, and C.-N.~J. Yu.
\newblock Cutting-plane training of structural svms.
\newblock {\em Machine Learning}, 77(1):27--59, 2009.

\bibitem{Kodirov2015}
E.~Kodirov, T.~Xiang, Z.~Fu, and S.~Gong.
\newblock Unsupervised domain adaptation for zero-shot learning.
\newblock In {\em ICCV}, 2015.

\bibitem{citeulike:7491128}
A.~Krizhevsky.
\newblock {Learning Multiple Layers of Features from Tiny Images}.
\newblock Master's thesis, 2009.

\bibitem{lampert2009attribute}
C.~H. Lampert, H.~Nickisch, and S.~Harmeling.
\newblock Learning to detect unseen object classes by between-class attribute
  transfer.
\newblock In {\em CVPR}, pages 951--958, 2009.

\bibitem{10.1109/TPAMI.2013.140}
C.~H. Lampert, H.~Nickisch, and S.~Harmeling.
\newblock Attribute-based classification for zero-shot visual object
  categorization.
\newblock {\em PAMI}, 36(3):453--465, 2014.

\bibitem{Li2015}
X.~Li and Y.~Guo.
\newblock Max-margin zero-shot learning for multi-class classification.
\newblock In {\em AISTATS}, 2015.

\bibitem{Li_ICCV2015}
X.~Li, Y.~Guo, and D.~Schuurmans.
\newblock Semi-supervised zero-shot classification with label representation
  learning.
\newblock In {\em ICCV}, 2015.

\bibitem{mahajan2011joint}
D.~Mahajan, S.~Sellamanickam, and V.~Nair.
\newblock A joint learning framework for attribute models and object
  descriptions.
\newblock In {\em ICCV}, pages 1227--1234, 2011.

\bibitem{mensink2014costa}
T.~Mensink, E.~Gavves, and C.~G.~M. Snoek.
\newblock Costa: Co-occurrence statistics for zero-shot classification.
\newblock In {\em CVPR}, pages 2441--2448, June 2014.

\bibitem{mensink2012metric}
T.~Mensink, J.~Verbeek, F.~Perronnin, and G.~Csurka.
\newblock Metric learning for large scale image classification: Generalizing to
  new classes at near-zero cost.
\newblock In {\em ECCV}, pages 488--501. 2012.

\bibitem{norouziMBSSFCD14}
M.~Norouzi, T.~Mikolov, S.~Bengio, Y.~Singer, J.~Shlens, A.~Frome, G.~S.
  Corrado, and J.~Dean.
\newblock Zero-shot learning by convex combination of semantic embeddings.
\newblock In {\em ICLR}, 2014.

\bibitem{palatucci2009zero}
M.~Palatucci, D.~Pomerleau, G.~E. Hinton, and T.~M. Mitchell.
\newblock Zero-shot learning with semantic output codes.
\newblock In {\em NIPS}, pages 1410--1418, 2009.

\bibitem{Parikh:2011:IBD:2191740.2191861}
D.~Parikh and K.~Grauman.
\newblock Interactively building a discriminative vocabulary of nameable
  attributes.
\newblock In {\em CVPR}, pages 1681--1688, 2011.

\bibitem{patterson2014sun}
G.~Patterson, C.~Xu, H.~Su, and J.~Hays.
\newblock The sun attribute database: Beyond categories for deeper scene
  understanding.
\newblock {\em IJCV}, 108(1-2):59--81, 2014.

\bibitem{ping2014marginal}
W.~Ping, Q.~Liu, and A.~Ihler.
\newblock Marginal structured svm with hidden variables.
\newblock In {\em ICML}, pages 190--198, 2014.

\bibitem{Qiao_2016_CVPR}
R.~Qiao, L.~Liu, C.~Shen, and A.~van~den Hengel.
\newblock Less is more: Zero-shot learning from online textual documents with
  noise suppression.
\newblock In {\em CVPR}, June 2016.

\bibitem{conf/nips/RohrbachES13}
M.~Rohrbach, S.~Ebert, and B.~Schiele.
\newblock Transfer learning in a transductive setting.
\newblock In {\em NIPS}, pages 46--54, 2013.

\bibitem{rohrbach2011largeScale}
M.~Rohrbach, M.~Stark, and B.~Schiele.
\newblock Evaluating knowledge transfer and zero-shot learning in a large-scale
  setting.
\newblock In {\em CVPR}, pages 1641--1648, 2011.

\bibitem{Romera-Paredes2015}
B.~Romera-Paredes and P.~H.~S. Torr.
\newblock An embarrassingly simple approach to zero-shot learning.
\newblock In {\em ICML}, 2015.

\bibitem{ILSVRCarxiv14}
O.~Russakovsky, J.~Deng, H.~Su, J.~Krause, S.~Satheesh, S.~Ma, Z.~Huang,
  A.~Karpathy, A.~Khosla, M.~Bernstein, A.~C. Berg, and L.~Fei-Fei.
\newblock {ImageNet Large Scale Visual Recognition Challenge}, 2014.

\bibitem{simonyan2014very}
K.~Simonyan and A.~Zisserman.
\newblock Very deep convolutional networks for large-scale image recognition.
\newblock {\em arXiv preprint arXiv:1409.1556}, 2014.

\bibitem{socher2013zero}
R.~Socher, M.~Ganjoo, C.~D. Manning, and A.~Ng.
\newblock Zero-shot learning through cross-modal transfer.
\newblock In {\em NIPS}, pages 935--943, 2013.

\bibitem{arXiv:1412.4564}
A.~Vedaldi and K.~Lenc.
\newblock Matconvnet -- convolutional neural networks for {MATLAB}.
\newblock {\em CoRR}, abs/1412.4564, 2014.

\bibitem{WahCUB_200_2011}
C.~Wah, S.~Branson, P.~Welinder, P.~Perona, and S.~Belongie.
\newblock {The Caltech-UCSD Birds-200-2011 Dataset}.
\newblock Technical report, 2011.

\bibitem{wang2016relational}
D.~Wang, Y.~Li, Y.~Lin, and Y.~Zhuang.
\newblock Relational knowledge transfer for zero-shot learning.
\newblock In {\em AAAI}, 2016.

\bibitem{wang2013unified}
X.~Wang and Q.~Ji.
\newblock A unified probabilistic approach modeling relationships between
  attributes and objects.
\newblock In {\em ICCV}, pages 2120--2127, 2013.

\bibitem{wu2014zero}
S.~Wu, S.~Bondugula, F.~Luisier, X.~Zhuang, and P.~Natarajan.
\newblock Zero-shot event detection using multi-modal fusion of weakly
  supervised concepts.
\newblock In {\em CVPR}, pages 2665--2672, 2014.

\bibitem{Xian2016Latent}
Y.~Xian, Z.~Akata, G.~Sharma, Q.~Nguyen, M.~Hein, and B.~Schiele.
\newblock Latent embeddings for zero-shot classification.
\newblock In {\em CVPR}, 2016.

\bibitem{yu2009learning}
C.-N.~J. Yu and T.~Joachims.
\newblock Learning structural svms with latent variables.
\newblock In {\em ICML}, pages 1169--1176. ACM, 2009.

\bibitem{yu2013designing}
F.~X. Yu, L.~Cao, R.~S. Feris, J.~R. Smith, and S.~F. Chang.
\newblock Designing category-level attributes for discriminative visual
  recognition.
\newblock In {\em CVPR}, pages 771--778, 2013.

\bibitem{yu2010attribute}
X.~Yu and Y.~Aloimonos.
\newblock Attribute-based transfer learning for object categorization with
  zero/one training example.
\newblock In {\em ECCV}, pages 127--140. 2010.

\bibitem{Zhang_2016_CVPR}
Y.~Zhang, B.~Gong, and M.~Shah.
\newblock Fast zero-shot image tagging.
\newblock In {\em CVPR}, June 2016.

\bibitem{Zhang2015}
Z.~Zhang and V.~Saligrama.
\newblock Zero-shot learning via semantic similarity embedding.
\newblock In {\em ICCV}, 2015.

\bibitem{zhang2016zero}
Z.~Zhang and V.~Saligrama.
\newblock Zero-shot learning via joint latent similarity embedding.
\newblock In {\em CVPR}, pages 6034--6042, 2016.

\bibitem{zhang2016ECCV}
Z.~Zhang and V.~Saligrama.
\newblock Zero-shot recognition via structured prediction.
\newblock In {\em ECCV}, 2016.

\end{thebibliography}
}

\end{document}